\documentclass{article}

\usepackage[final]{neurips_2025}

\usepackage{yfonts}
\usepackage{amsmath,bm}

\usepackage[utf8]{inputenc} 
\usepackage[T1]{fontenc}    
\usepackage{hyperref}       
\usepackage{url}            
\usepackage{booktabs} 
\usepackage{diagbox} 
\usepackage{amsfonts}       
\usepackage{nicefrac}       
\usepackage{microtype}      
\usepackage{xcolor}         
\usepackage{amsmath,amsthm} 
\usepackage{adjustbox}
\usepackage{enumitem}
\usepackage{multirow}
\usepackage{rotating}
\usepackage{yfonts}
\usepackage{hyperref}
\usepackage{url}

\usepackage[utf8]{inputenc} 
\usepackage[T1]{fontenc}    
\usepackage{hyperref}       
\usepackage{url}            
\usepackage{booktabs}       
\usepackage{amsfonts}       
\usepackage{nicefrac}       
\usepackage{microtype}      
\usepackage{xcolor}         

\usepackage{comment}
\usepackage{amsmath}
\usepackage{graphicx}
\usepackage{subcaption}
\usepackage[ruled,vlined,linesnumbered,noend]{algorithm2e}
\usepackage{amsmath}
\usepackage{amssymb}
\usepackage{mathtools}
\usepackage{amsthm}
\usepackage{pgfplots}   
\usepackage{microtype}
\usepackage{graphicx}
\usepackage{booktabs} 
\pgfplotsset{compat=1.18} 
\usepackage{multirow}

\usepackage{enumitem}
\usepackage{accents} 
\usepackage{caption}
\usepackage{graphicx}

\usepackage[mathscr]{euscript}
\usepackage{cleveref}
\usepackage{wrapfig}

\usepackage{nicematrix}

\usepackage[utf8]{inputenc} 
\usepackage[T1]{fontenc}    
\usepackage{hyperref}       
\usepackage{url}            
\usepackage{booktabs}       
\usepackage{amsfonts}       
\usepackage[ruled,vlined,linesnumbered,noend]{algorithm2e}
\usepackage{nicefrac}       
\usepackage{microtype}      
\usepackage{xcolor}         
\usepackage{amsmath}
\usepackage{amssymb}
\usepackage{mathtools}
\usepackage{amsthm}
\newtheorem{theorem}{Theorem}

\newcommand{\augU}{{\widehat{U}}}
\newcommand{\augS}{{\widehat{S}}}

\newcommand{\augV}{{\widehat{V}}}

\newcommand{\rom}[1]{\uppercase\expandafter{\romannumeral #1\relax}}

\newcommand{\fx}{\mathbf{x}}

\newcommand{\fz}{\mathbf{z}}

\def\ALGNAME{ALG}

\title{A geometric framework for momentum-based optimizers for low-rank training}

\author{%
  Steffen Schotth\"ofer\thanks{Computer Science and Mathematics Division; Oak Ridge National Laboratory; 
  Oak Ridge, TN 37831 USA; Mail correspondence: \texttt{schotthofers@ornl.gov}}\,\,\,\,
 Timon Klein\thanks{Department of Mathematics;
Otto von Guericke University Magdeburg;
39106 Magdeburg; 
Germany}\,\,\,\, and\,\,\, 
  Jonas Kusch\thanks{Scientific Computing; Norwegian University of Life Sciences; 
Drøbakveien 31, 1433 Ås; Norway}
}

\begin{document}

\maketitle

\begin{abstract}
 Low-rank pre-training and finetuning have recently emerged as promising techniques for reducing the computational and storage costs of large neural networks. Training low-rank parameterizations typically relies on conventional optimizers such as heavy ball momentum methods or Adam. In this work, we identify and analyze potential difficulties that these training methods encounter when used to train low-rank parameterizations of weights. In particular, we show that classical momentum methods can struggle to converge to a local optimum due to the geometry of the underlying optimization landscape. To address this, we introduce novel training strategies that combine dynamical low-rank approximation with momentum-based optimization, explicitly accounting for the intrinsic geometry of the parameter space. We validate our methods through numerical experiments, demonstrating stronger validation metrics at given parameter budgets.
\end{abstract}


\section{Introduction}
Deep learning models have achieved remarkable success across natural language processing and computer vision tasks, but their deployment remains computationally expensive due to the large number of trainable parameters. To address this, parameter-efficient strategies have been developed to reduce memory and compute requirements during training. Common approaches include sparsification \cite{guo2016dynamic, molchanov2017pruning, he2017channel}, quantization \cite{wu2016quantized, courbariaux2016binarized}, and layer factorization. The latter has gained considerable attention for pre-training \cite{wang2021pufferfish, khodak2021initialization, schotthofer2022low, schotthöfer2024federateddynamicallowranktraining, zangrando2023rank} and especially finetuning \cite{hu2021lora, valipour2023dylora, zhang2023adalora, hayou2024lora, zhao2024galore, lialin2024relora, schotthöfer2024GeoLoRAgeometricintegrationparameter}. Layer factorization represents weights (or adapters) as low-rank matrices, allowing only the low-rank factors to be trained. This significantly reduces both memory usage and computational cost.

In the class of low-rank layer factorizations, one of the most popular methods is LoRA \citep{hu2021lora}, which applied typical optimizers such as stochastic gradient descent (SGD) or Adaptive Moment Estimation (Adam) \cite{adam} directly to the low-rank factors. The combination of these optimizers with LoRA does not guarantee an optimal optimization trajectory \cite{schotthöfer2024GeoLoRAgeometricintegrationparameter,zhang2023adalora}. To overcome the former challenge, various improvements have been proposed: For example,   LoRA+ \cite{hayou2024lora} proposes the use of separate learning rates for different components of the low-rank decomposition, Dora \citep{Mao2024DoRAEP} normalizes the factor matrices and introduces a magnitude parameter. Furthermore, to overcome the challenge of tuning the rank of the LoRA ansatz, AdaLoRA \cite{zhang2023adalora}  adaptively allocating the parameter budget during training, by masking off rows and columns of the adapter matrices.
Other low-rank methods focus entirely on the optimizer states: GaLore \cite{zhao2024galore} projects full-rank weight gradients into a low-rank subspace to reduce the memory footprint of the optimizer state. Tensor-GaLore \cite{tensorgalore} generalizes this technique to high-order tensor-parameterized models, further improving efficiency for large-scale architectures. Adafactor~\cite{adafactor} approximates the second-moment matrix using a low-rank decomposition of Adam with low-rank factors.

A particularly promising class of layer factorization strategies is dynamical low-rank training (DLRT) which has been introduced in \cite{schotthofer2022low} and has since been used in various tasks \cite{zangrando2023rank,schotthöfer2024federateddynamicallowranktraining,schotthöfer2024GeoLoRAgeometricintegrationparameter,coquelin2024harnessing,kusch2025augmentedbackwardcorrectedprojectorsplitting}. 
DLRT projects the gradient flow dynamics onto the tangent space of the manifold of low-rank parameters, thereby achieving convergence guarantees to low-rank optimal weights \cite{schotthöfer2024GeoLoRAgeometricintegrationparameter}. However, these projections are inherently non-smooth, leading to ill-conditioned optimization landscapes and requiring smaller learning rates for stable training \cite{schotthofer2022low}. To ensure robust integration of such projections, DLRT constrains movement to flat subspaces within the low-rank manifold, enabling stable convergence to low-rank optima. The method is rank adaptive - using a basis augmentation and subsequent singular value-based truncation criterion to adapt the rank of the low-rank factorization. It further enables extensions to increase adversarial robustness of the compressed neural networks by enforcing orthonormality of the low-rank bases and projecting onto a well-conditioned manifold \cite{savostianova2023robustlowranktrainingapproximate} or regularizing the condition number of the factor matrix \cite{schotthöfer2025dynamicallowrankcompressionneural}.

Despite its efficiency, current DLRT approaches primarily rely on SGD, and it remains unclear how adaptive optimizers such as Adam \cite{adam} can be effectively applied. This is a significant limitation, as many state-of-the-art models depend on momentum-based optimizers like Adam and its variants for performance and stability. Therefore, the extension of DLRT to such optimization techniques is crucial.

To bridge this gap, we introduce a novel momentum-based optimization framework for low-rank pretraining and finetuning. The method integrates adaptive momentum techniques into the framework of dynamical low-rank training (DLRT), preserving the low-rank structure of model weights while enabling stable and efficient updates. This establishes a link between low-rank optimization and adaptive gradient methods, yielding both theoretical insights and practical improvements. Beyond the method derivation, we analyze why LoRA-style adapters—and low-rank parameterizations more broadly require momentum-based optimizers that are aware of the geometry of the low-rank parameter space, i.e., the underlying manifold structure. Naively applying standard optimizers such as heavy ball to low-rank parameterizations can produce updates that do not correspond to a gradient flow leading to a low-rank optimum. As a result, these methods may fail to converge to valid low-rank solutions. We show how DLRT can be adapted to approximate geometry-respecting gradient flows that consistently drive convergence toward low-rank optima. Together, these contributions lay a foundation for more robust and efficient training of large-scale models under low-rank constraints.

Compared to low-rank methods that act on the optimizer states only \citep{zhao2024galore,tensorgalore,adafactor} we provide a holistic interpretation of a low-rank optimization algorithm that adaptively compresses the network weights, gradients and optimizer states simultaneously, achieving superior compression performance at high validation metrics. We remark that the method is directly extendable for tensor-valued neural networks using, e.g., low-rank Tucker factorization.

The paper is structured as follows: We first discuss limitations of naive momentum methods and show how to adapt the underlying gradient flow to achieve convergence to a low-rank optimum in Section~\ref{sec:momentum-low-rank}. While the adapted gradient flow facilitates convergence, constructing robust numerical optimizers from it is challenging due to its inherent stiffness. We propose a low-rank heavy ball optimizer in Section~\ref{sec:DLRT-heavy-ball} which integrates the adapted gradient flow robustly. In Section~\ref{sec:low-rank-adam}, we construct a fully low-rank Adam optimizer by leveraging insights gained in Section~\ref{sec:DLRT-heavy-ball}. In Section~\ref{sec:num_res} we underline the efficiency of the proposed method through numerical experiments. In particular, we demonstrate fast convergence and superior validation accuracy at high compression levels for training from scratch, transfer learning, and low-rank finetuning of different neural network architectures and benchmarks.

\section{Momentum-based low-rank training}\label{sec:momentum-low-rank}
We consider a low-rank neural network of the form
\begin{align*} 
\mathcal{N}(x) = \sigma_L(U_LS_LV_{L}^\top z_{L-1}(x)) \, ,
\end{align*} 
where $z_{L-1}(x)$ is defined recursively by
\begin{align}\label{eq:model}
    z_0(x) = x \in \mathbb{R}^{n_0} \, ,\quad\text{and}\quad
    z_l(x) = \sigma_l(U_lS_lV_{l}^\top z_{l-1}(x)) \in \mathbb{R}^{n_l}, \quad \forall l &= 1, \dots, L \,.
\end{align} 
Here, the weight matrices are defined as $W_{l} := U_l S_l V_{l}^\top \in \mathbb{R}^{n_{l} \times n_{l-1}}$, where $U_l \in \mathbb{R}^{n_{l} \times r_{l}}$ and $V_l \in \mathbb{R}^{n_{l-1} \times r_{l}}$ are orthonormal low-rank factors, and $S_l \in \mathbb{R}^{r_{l} \times r_{l}}$ is the coefficient matrix. Thus, $W_l$ lies in the manifold of rank $r_l$ matrices which we denote by $\mathcal{M}_{r_l}$. Additionally, $\sigma_{l}$ represents the activation function of layer $l$.  For simplicity of notation, we do not consider biases, but a model with biases can always be expressed as {Eq.} \eqref{eq:model} by folding biases into weights and creating an input dimension that is always one. 
Several methods have been proposed to train the low-rank weight matrices $W_{l}$ to minimize a given loss function $\mathcal{L}$. Among these, the simplest training rule is the steepest descent method, which, for a fixed\footnote{We restrict the discussion to a single layer without loss of generality, following the arguments of \cite[Appendix~I]{schotthöfer2024GeoLoRAgeometricintegrationparameter}.} layer $l$ with weights $W = USV^{\top} \equiv W_l$ reads
    \begin{align}\label{eq:SD}
    U^{n+1} = U^n - \lambda \nabla_U \mathcal{L}^n \,, \quad 
    S^{n+1} = S^n - \lambda \nabla_S \mathcal{L}^n\,, \quad 
    V^{n+1} =V^n - \lambda \nabla_V \mathcal{L}^n\,,
\end{align}
where $\lambda$ is the learning rate, the index $n$ denotes the training iteration, and we have used the shorthand notation $\mathcal{L}^n := \mathcal{L}(U^nS^nV^{n,\top})$. A more general framework that facilitates the numerical analysis and construction of novel numerical methods interprets the steepest descent method as an explicit Euler time discretization of the continuous gradient flow equations
\begin{align*}
    \dot U(t) = - \nabla_U \mathcal{L}\,, \quad
    \dot S(t) = - \nabla_S \mathcal{L}\,, \quad
    \dot V(t) = - \nabla_V \mathcal{L} \,,
\end{align*}
where $\mathcal{L} := \mathcal{L}(U(t)S(t)V(t)^{\top})$. Then, the steepest descent update equation for the individual factors can be obtained through a forward Euler discretization of the pseudo-time $t$. E.g., the update equation for $U$ can be retrieved through $\dot U(t_n) \approx \frac{1}{\lambda} (U^{n+1} - U^n)$ and $\mathcal{L} \approx \mathcal{L}^n$. While steepest descent methods offer a simple strategy to drive the parameters to a locally optimal point, one of the most widely adopted strategies for updating factorized parameters $W$, along with their factorized momentum terms $\mathcal{V} = U_{\mathcal{V}} S_{\mathcal{V}} V_{\mathcal{V}}^{\top}$, involves momentum-based optimization. For instance, in the case of the heavy ball method applied to all low-rank factors individually, e.g., in LoRA \cite{hu2021lora}, the associated update equations take the form
\begin{subequations}\label{eq:heavy_ball_updates}
\begin{alignat}{2}
    U^{n+1} = U^n + \lambda\, U_{\mathcal{V}}^n,\quad U_{\mathcal{V}}^{n+1} &= (1 - \lambda\gamma)\, U_{\mathcal{V}}^n - \lambda\, \nabla_U \mathcal{L}^n,\quad &&\nabla_U \mathcal{L}^n = (\nabla_W \mathcal{L}^n) V^n (S^n)^\top, \\
    V^{n+1} = V^n + \lambda\, V_{\mathcal{V}}^n,\quad V_{\mathcal{V}}^{n+1} &= (1 - \lambda\gamma)\, V_{\mathcal{V}}^n - \lambda\, \nabla_V \mathcal{L}^n,\quad &&\nabla_V \mathcal{L}^n = (\nabla_W \mathcal{L}^n)^\top U^n S^n, \\
    S^{n+1} = S^n + \lambda\, S_{\mathcal{V}}^n,\quad S_{\mathcal{V}}^{n+1} &= (1 - \lambda\gamma)\, S_{\mathcal{V}}^n - \lambda\, \nabla_S \mathcal{L}^n,\quad &&\nabla_S \mathcal{L}^n = (U^n)^\top (\nabla_W \mathcal{L}^n) V^n.
\end{alignat}
\end{subequations}
The associated gradient flow equations are given by:
\begin{subequations}\label{eq:vanilla_grad_flow}
\begin{alignat}{2}
    \dot{U} = U_{\mathcal{V}}\,,\quad \dot{U}_{\mathcal{V}}  + \gamma U_{\mathcal{V}} + \nabla_U\mathcal{L} =\,& 0\,, \quad &&\nabla_U\mathcal{L} = (\nabla_W\mathcal{L})VS^{\top}\,, \label{eq:vanilla_grad_flow_U}\\
    \dot{V} = V_{\mathcal{V}}\,,\quad \dot{V}_{\mathcal{V}}  + \gamma V_{\mathcal{V}} + \nabla_V\mathcal{L} =\,& 0\,, \quad  &&\nabla_V\mathcal{L} = (\nabla_W\mathcal{L})^{\top}US\,, \label{eq:vanilla_grad_flow_V}\\
    \dot{S} = S_{\mathcal{V}}\,,\quad \dot{S}_{\mathcal{V}}  + \gamma S_{\mathcal{V}} + \nabla_S\mathcal{L} =\,& 0\,, \quad  &&\nabla_S\mathcal{L}= U^{\top}\nabla_W\mathcal{L}V\,,
\end{alignat}
\end{subequations}
{where $\gamma$ denotes the momentum decay parameter.} While these equations are optimal when treating all low-rank factors in isolation, they do not account for the fact that factors change simultaneously. For example, {Eq.} \eqref{eq:vanilla_grad_flow_U} is expected to drive the solution to an optimum, only if $S(t)$ and $V(t)$ remain constant in time. When accounting for the dynamics of all three low-rank factors, the resulting gradient flow for $W(t) = U(t)S(t)V(t)^{\top}$ takes the form
\begin{subequations}
    \begin{align}\label{eq:naiveUppdateFlow}
    \dot W &=\, U_{\mathcal{V}}SV^{\top} + US_{\mathcal{V}}V^{\top} + USV_{\mathcal{V}}^{\top}\,,\quad\text{and}\\
    \dot{\mathcal{V}} + 3\gamma \mathcal{V} &=\, - \nabla_U\mathcal{L}S_{\mathcal{V}}V_{\mathcal{V}}^{\top} - U_{\mathcal{V}}\nabla_S\mathcal{L}V_{\mathcal{V}}^{\top} - U_{\mathcal{V}} S_{\mathcal{V}}\nabla_V\mathcal{L}^{\top}\,\nonumber\\
    &=\, - (\nabla_W\mathcal{L})VS^{\top} S_{\mathcal{V}}V_{\mathcal{V}}^{\top} - U_{\mathcal{V}}U^{\top}\nabla_W\mathcal{L}VV_{\mathcal{V}}^{\top} - U_{\mathcal{V}} S_{\mathcal{V}}S^{\top}U^{\top}\nabla_W\mathcal{L}\,\nonumber\\
    &=:\, -\widehat P(W, \mathcal{V}) \nabla_W\mathcal{L}\,. \label{eq:naiveUppdateFlowMomentum}
\end{align}
\end{subequations}
This can easily be shown with the product rule, e.g., $\dot W = \dot USV^{\top} + U\dot SV^{\top} + US\dot V^{\top}$ and plugging in time derivatives from {Eq.} \eqref{eq:vanilla_grad_flow}. These evolution equations for $W$ and $\mathcal{V}$ are fundamentally different from the momentum-based gradient flow equations of the full-rank problem
\begin{align}\label{eq:grad_flow_full}
    \dot W_{\mathrm{full}} = \mathcal{V}_{\mathrm{full}} \,,\quad \dot{\mathcal{V}}_{\mathrm{full}}   + \gamma \mathcal{V}_{\mathrm{full}}  =- \nabla_W\mathcal{L}(W_{\mathrm{full}})\,.
\end{align}
Indeed, a proper formulation of {Eq.} \eqref{eq:grad_flow_full} that drives the weights of a heavy ball method to a local optimum while preserving the low-rank representation of $W$, requires that 
\begin{align}\label{eq:gradflowopt}
    \dot W  = P(W)\mathcal{V} \,,\quad
    \dot{\mathcal{V}}+\gamma \mathcal{V} = - P(W)\nabla_W\mathcal{L} \,,
\end{align}
\begin{wrapfigure}{R}{0.4\textwidth}
\vspace{-1.8em}
    \centering
    \includegraphics[width=\linewidth]{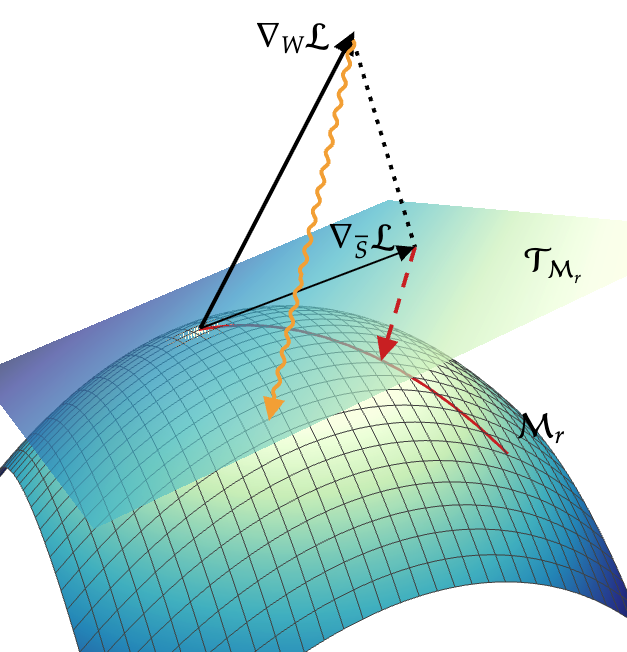}
      \caption{Geometric interpretation of \Cref{alg:heavy_ball_dlrt}. We compute the parametrization of the tangent plane $\mathcal{T}_{\mathcal{M}_r}$. Then, we compute the projected gradient  $\nabla_{\bar{S}}\mathcal{L}$ to construct the low-rank momentum update. The momentum optimizer is then applied to the low-rank weight coefficient $\widehat{S}$.
      Lastly, we retract the updated coefficients back onto the manifold $\mathcal{M}_r$. The interpretation of \Cref{alg:adam_dlrt} is analogous.      
      LoRA-like methods do not employ orthogonal projections onto $\mathcal{T}_{\mathcal{M}_r}$, but instead map the full gradient $\nabla_W\mathcal{L}$ implicitly onto $\mathcal{M}_r$. The linear map (displayed as the wavy orange line) may map the gradient direction far away from the properly projected gradient flow, leading to suboptimal descent directions.}
    \label{fig_manifold}
    \vspace{-5em}
\end{wrapfigure}
where for $W = USV^{\top}$, the projector onto the tangent space is given by $P(W)Z := U U^{\top}Z(I-VV^{\top}) + ZV V^{\top}$, see \cite[Lemma~4.1]{koch2007dynamical} and \Cref{fig_manifold} for geometric interpretation. Note that in abuse of notation, we have recycled $W$ and $\mathcal{V}$ here to denote the weights and momentum terms following {Eq.} \eqref{eq:gradflowopt} instead of the naive {Eq.} \eqref{eq:vanilla_grad_flow}.  Then, the time evolution will drive $W$ into a low-rank steady state $(W^{\star}, \mathcal{V}^{\star})$ such that $P(W^{\star})\nabla_W\mathcal{L}(W^{\star}) = 0$, see Theorem~\ref{th:lr_conv}. This steady state thus fulfills the optimality condition of a local optimum, see, e.g. \cite[Theorem~3.4]{sato2021riemannian}. Such a condition is not ensured by the simultaneous descent equations of \eqref{eq:vanilla_grad_flow} since, in general, $\dot W \neq P(W)\mathcal{V}$ and $\widehat P(W, \mathcal{V})\nabla_W\mathcal{L} \neq P(W)\nabla_W\mathcal{L}${, see Theorem~\ref{th:failedConv}}. Therefore, training low-rank factors with conventional momentum methods does not necessarily ensure convergence to a low-rank optimum. 

We aim to derive a numerical method that is consistent with the optimal gradient-flow equations \eqref{eq:gradflowopt}. A central limitation of {Eq.} \eqref{eq:gradflowopt} is that it does not preserve the low-rank structure of the momentum term $\mathcal{V}$, thus leading to prohibitive computational costs and memory requirements. Instead, we aim to derive a method that fulfills 
\begin{align}\label{eq:momentum_efficient}
   \dot{\mathcal{V}} +\gamma \mathcal{V}=  - P(\mathcal{V})\nabla_W\mathcal{L} \,
\end{align}
which preserves the low-rank structure of the momentum term. Indeed, the factorized solution of {Eq.} \eqref{eq:momentum_efficient} fulfills (see Theorem~\ref{th:dlra_conv})
\begin{subequations}\label{eq:grad_flow_dlr}
\begin{alignat}{2}
    \dot{U}_{\mathcal{V}}  =\,& - (I-U_{\mathcal{V}}U_{\mathcal{V}}^{\top})\nabla_W\mathcal{L}V_{\mathcal{V}}S_{\mathcal{V}}^{-1}\,, \\
    \dot{V}_{\mathcal{V}}  =\,& - (I-V_{\mathcal{V}}V_{\mathcal{V}}^{\top})\nabla_W\mathcal{L}^{\top}U_{\mathcal{V}}S_{\mathcal{V}}^{-\top}\,, \\
    \dot{S}_{\mathcal{V}}  =\,& -\gamma S_{\mathcal{V}} - U_{\mathcal{V}}^{\top}\nabla_W\mathcal{L}V_{\mathcal{V}} \,\,,\label{eq:grad_flow_dlr_S}
\end{alignat}
\end{subequations}
with initial condition $U(0) = U_{\mathcal{V}}(0)$ and $V(0) = V_{\mathcal{V}}(0)$. While this formulation and in particular {Eq.} \eqref{eq:gradflowopt} provide a good basis for constructing numerical methods, it also introduces the inverse terms $S_{\mathcal{V}}^{-1}$ and $S_{\mathcal{V}}^{-\top}$ on the right-hand side, rendering the system highly stiff, especially when these matrices are ill-conditioned. This stiffness can be treated through robust time integrators \cite{ceruti2021rank,ceruti2024parallel} developed in the field of dynamical low-rank approximation \cite{koch2007dynamical} which have also been used for stochastic-gradient descent methods in dynamical low-rank training \cite{schotthofer2022low,zangrando2023rank,schotthöfer2024federateddynamicallowranktraining,schotthöfer2024GeoLoRAgeometricintegrationparameter,Hnatiuk,schotthöfer2025dynamicallowrankcompressionneural}.  In the following section, we formulate an algorithm that provably approximates the gradient flow of {Eq.} \eqref{eq:grad_flow_dlr} by following ideas of \cite{ceruti2021rank}. It turns out that this method is a consistent approximation of the optimal gradient flow of {Eq.} \eqref{eq:gradflowopt} under mild assumptions.

\section{A low-rank heavy ball method}\label{sec:DLRT-heavy-ball}
To approximate {Eq.} \eqref{eq:gradflowopt}, we start with the first time step from $t_0 = 0$ to $t_1 = \lambda$ and note that by definition of the initial condition $U^0 := U(t_0) = U_{\mathcal{V}}(t_0)$ and $V^0 := V(t_0) = V_{\mathcal{V}}(t_0)$. Let us first construct an augmented basis to ensure that the range and co-range of $\mathcal{V}$ are fully spanned. To ensure robustness to small singular values, we introduce a change of variables and evolve the basis along $K_{\mathcal{V}} (t) = U_{\mathcal{V}}(t)S_{\mathcal{V}}(t)$ and $L_{\mathcal{V}} (t) = V_{\mathcal{V}}(t)S_{\mathcal{V}}(t)^{\top}$ for $t\in[t_{0}, t_{1}]$ while keeping $V^{0} = V_{\mathcal{V}}(t_0)$ and $U^{0} = U_{\mathcal{V}}(t_0)$ fixed, respectively \cite{ceruti2021rank}. Then, using the product rule and derivatives from {Eq.} \eqref{eq:grad_flow_dlr}, we get
\begin{alignat*}{2}
    \dot{K}_{\mathcal{V}} (t) = \dot{U}_{\mathcal{V}} (t)S_{\mathcal{V}}(t) + U_{\mathcal{V}}(t)\dot S_{\mathcal{V}}(t) \stackrel{\eqref{eq:grad_flow_dlr}}{=}\,& - \nabla_W\mathcal{L}(W(t))V^{0}\,,\quad &&K_{\mathcal{V}} (t_{0}) = U^{0}S^{0}_v\,, \\
    \dot{L}_{\mathcal{V}} (t) =\dot{V}_{\mathcal{V}} (t)S_{\mathcal{V}}(t)^{\top} + V_{\mathcal{V}}(t)\dot S_{\mathcal{V}}(t)^{\top} \stackrel{\eqref{eq:grad_flow_dlr}}{=}\,& -\nabla_W\mathcal{L}(W(t))^{\top}U^{0}\,,\quad &&L_{\mathcal{V}} (t_{0}) = V^{0}S^{0,\top}_v\,.
\end{alignat*}
\begin{algorithm}[t]
\DontPrintSemicolon
\SetAlgoLined
\SetKwInOut{Input}{Input}
\SetKwComment{Comment}{$\triangleright$\ }{}

\Input{Initial orthonormal bases $U,V\in\mathbb{R}^{n\times r}$ and coefficients $S, S_{\mathcal{V}}\in\mathbb{R}^{r\times r}$;\;
$\tau$: singular value threshold for rank truncation;\;
$\lambda$: learning rate.
}
Evaluate $\mathcal{L}(USV^\top)$\tcc*{Forward evaluate}
$G_U\gets\nabla_{U}\mathcal{L}(USV^\top);\,G_{\mathcal{V}}\gets\nabla_{V}\mathcal{L}(USV^\top)$ \tcc*{Backprop}

$\left\{
\begin{array}{l}
\widehat U\gets\texttt{  basis\_augmentation}(U, G_U) \\
\widehat V\gets\texttt{  basis\_augmentation}(V, G_{\mathcal{V}})
\end{array}
\right.${\tcc*{in parallel}}

$\bar S \gets \widehat U^{\top}USV^\top\widehat V; \bar S_{\mathcal{V}} \gets \widehat U^{\top}US_{\mathcal{V}} V^\top\widehat V$

Evaluate $\mathcal{L}(\widehat U \bar S \widehat V^\top)$\tcc*{Forward evaluate}
$G_S\gets\nabla_{\bar S}\mathcal{L}(\widehat U \bar S \widehat V^\top)$ \tcc*{Backprop}
$\widehat S_{\mathcal{V}} \gets (1-\gamma)\bar S_{\mathcal{V}} - \lambda G_S;\; \widehat S \gets \bar S + \lambda \widehat S_{\mathcal{V}}
$\tcc*{coefficient update} 
$U,S,V, S_{\mathcal{V}} \gets ${\tt truncation}$(\widehat S, \widehat S_{\mathcal{V}}, \widehat U, \widehat V; \tau  )$

\caption{Single iteration of the dynamical low-rank momentum method. \\ The functions \texttt{basis\_augmentation}, and \texttt{truncation} are detailed in \Cref{alg_helper} in the appendix. }\label{alg:heavy_ball_dlrt}
\end{algorithm}
As no ill-conditioned $S_{\mathcal{V}}^{-1}$ terms affect the dynamics, one can use a forward Euler step to update $K_{\mathcal{V}} $ and $L_{\mathcal{V}} $ from $t_{0}$ to the next time step $t_{1}$. Thus, we get for $K_{\mathcal{V}} ^n \approx K_{\mathcal{V}} (t_n)$, $L_{\mathcal{V}} ^n \approx L_{\mathcal{V}} (t_n)$, and $W^n = W(t_n)$
\begin{alignat*}{2}
    K_{\mathcal{V}} ^{1} =\,& K_{\mathcal{V}} ^0 - \lambda\nabla_W\mathcal{L}(W^0)V^{0}\,,\quad && \text{ with }K_{\mathcal{V}} ^0 = U^{0}S^{0}_v\,, \\
    L_{\mathcal{V}} ^{1} =\,& L_{\mathcal{V}} ^{0} - \lambda\nabla_W\mathcal{L}(W^0)^{\top}U^{0}\,,\quad && \text{ with } L_{\mathcal{V}} ^0 = V^{0}S^{0,\top}_v\,.
\end{alignat*}
Denoting an orthonormalization algorithm like Gram-Schmidt as $\text{ortho}$, $K_{\mathcal{V}} ^{1}$ and $L_{\mathcal{V}} ^{1}$ are therefore spanned by
\begin{align*}
    \widehat U = \text{ortho}(U^{0}, \nabla_W \mathcal{L}(W^{0}) V^{0} )\,,\quad \widehat V = \text{ortho}(V^{0}, (\nabla_W \mathcal{L}(W^{0}))^{\top} U^{0} )\,.
\end{align*}
Following \cite[Cor.~2.2]{NEURIPS2024_ea48cb23}, we note that $\nabla_U \mathcal{L} = \nabla_W \mathcal{L}VS^{\top}$ and $\nabla_V \mathcal{L} = (\nabla_W \mathcal{L})^{\top}US$, meaning that $\widehat U$ and $\widehat V$ can be rewritten as
\begin{align*}
    \widehat U = \text{ortho}(U^{0}, \nabla_U \mathcal{L}(W^{0}))\,,\quad \widehat V = \text{ortho}(V^{0}, \nabla_V \mathcal{L}(W^{0})) \,.
\end{align*}
We note here that this choice of the updated bases for $\mathcal{V}$ also approximates the updated parameters $W(t_1)$ of the equation $\dot W = \mathcal{V}$. With an implicit Euler time discretization, we have
\begin{align*}
    W^1 = U^0S^0V^{0,\top} + \lambda \widehat U \widehat S_{\mathcal{V}}^1 \widehat V^{\top} = \widehat U ( \widehat U^{\top} U^0S^0V^{0,\top}\widehat V + \lambda  \widehat S_{\mathcal{V}}^1) \widehat V^{\top} =: \widehat U \widehat S^1 \widehat V^{\top}\,,
\end{align*}
where $\widehat S^1 := \widehat U^{\top} U^0S^0V^{0,\top}\widehat V + \lambda  \widehat S_{\mathcal{V}}^1$ and $\widehat S_{\mathcal{V}}^1$ is the time-updated coefficient matrix of $\mathcal{V}$ which we will derive in the following: Solving {Eq.} \eqref{eq:grad_flow_dlr_S} with fixed bases $\widehat U$ and $\widehat V$ yields
\begin{align*}
    \dot{S}_{\mathcal{V}} (t) = -\gamma S_{\mathcal{V}}(t) - \widehat U^{\top}\nabla_W\mathcal{L}(W(t))\widehat V\, \qquad S_{\mathcal{V}}(t_{0}) = \widehat U^{\top} U^{0}S_{\mathcal{V}}^{0}V^{0,\top}\widehat V\,.
\end{align*}
Here, we choose $S_{\mathcal{V}}(t_{0})$ as the coefficient matrix $S_{\mathcal{V}}^{0}$ projected to the updated bases $\widehat U$ and $\widehat V$. This choice is crucial as it ensures the momentum term to be spanned with the updated basis. Using a forward Euler time discretization yields
\begin{alignat*}{2}
    \widehat S_{\mathcal{V}}^{1} =\,& (1-\gamma)\widehat U^{\top}U^{0}S_{\mathcal{V}}^{0}V^{0,\top}\widehat V - \lambda\widehat U^{\top}\nabla_W \mathcal{L}(U^{0}S^{0}V^{0,\top})\widehat V\,.
\end{alignat*}

Let us note that with $\bar S := \widehat U^{\top} U^{0}S^{0}V^{0,\top}\widehat V$ we have
\begin{align*}
    \widehat U^{\top}\nabla_W \mathcal{L}(U^{0}S^{0}V^{0,\top})\widehat V = \nabla_{\bar S} \mathcal{L}(\widehat U\bar S\widehat V^{\top})\,.
\end{align*}
Thus, with $\bar S_{\mathcal{V}} := \widehat U^{\top} U^{0}S_{\mathcal{V}}^{0}V^{0,\top}\widehat V$, the final coefficient updates (including the update for $S$) are  
\begin{alignat*}{2}
    \widehat S_{\mathcal{V}}^{1} =\,& (1-\gamma)\bar S_{\mathcal{V}} - \lambda\nabla_{\bar S} \mathcal{L}(\widehat U\bar S\widehat V^{\top})\,\\
    \widehat S^{1} =\,& \bar S + \lambda\widehat S_{\mathcal{V}}^{1}\,.
\end{alignat*}
We note that the basis for $W$ and $\mathcal{V}$ remain identical after one time update, thus the above derivation holds for general time updates from $t_n$ to $t_{n+1}$. Since, by construction, the updated bases $\widehat U$ and $\widehat V$ have doubled in rank compared to $U^{n}$ and $V^{n}$, we perform a truncation step. The truncation can be performed back to the original rank $r$, or formulated with a relative truncation threshold for a given tolerance parameter $\vartheta=\tau \|\augS\|$ \cite{ceruti2021rank} or a rank budget \cite{zhang2023adalora}, enabling a rank adaptive method.

One step of the resulting method using a relative truncation threshold is summarized in Algorithm~\ref{alg:heavy_ball_dlrt}. A main distinction of this method from a naive application of a heavy ball method to DLRT \cite{schotthofer2022low} is that 1) our method uses the same bases for parameters and momentum terms, 2) our method does not use a momentum method to update the bases, but uses a classical basis augmentation instead, and 3) the method projects momentum terms onto the new bases after the basis update. These three choices ensure that the method approximates the optimal gradient flow of {Eq.} \eqref{eq:gradflowopt} independent of the condition number of $S^{-1}$ and $S_{\mathcal{V}}^{-1}$ when the truncation tolerance is sufficiently small, which we make rigorous in 
\Cref{th:robust_error_bound}.

\Cref{th:robust_error_bound} shows that the low-rank momentum method produces solutions that are close to the solutions of full-rank (baseline) trained neural networks. Thus,  we expect the validation accuracy of the trained low-rank networks to match the full-rank baseline. This is empirically confirmed in \Cref{tab_lr_pretrain_summary}.

\section{A low-rank Adam method}\label{sec:low-rank-adam}

While heavy ball methods are the cornerstone of momentum-based optimization, they are commonly outperformed by momentum-based optimization methods that include stepsize control. Among these, perhaps the most popular method is Adaptive Moment Estimation (Adam), which was introduced in \cite{adam} and has significantly impacted the machine learning community; see \Cref{alg_adam} for a reference formulation of Adam.  While Adam exhibits superior performance, the non-linearities it introduces and a missing gradient flow formulation make a rigorous derivation of an extension to LoRA-type training difficult. In the following, we use the insights gained from the heavy ball method to construct a low-rank Adam optimizer. Adam's main distinction from heavy ball methods is the adaptive stepsize control, which is determined from the exponentially weighted moving average of the 
first and second moment of the gradient, denoted by {$\mathcal{V}$} and {$\mathcal{K}$}. 
\begin{algorithm}[t]
\DontPrintSemicolon
\SetAlgoLined
\SetKwInOut{Input}{Input}
\SetKwComment{Comment}{$\triangleright$\ }{}

\Input{Initial orthonormal bases $U,V\in\mathbb{R}^{n\times r}$ and coefficients $S, S_{\mathcal{V}}, S_{\mathcal{K}}\in\mathbb{R}^{r\times r}$;\;
$\tau$: singular value threshold for rank truncation;\;
$\lambda$: learning rate;\;
$\beta_1, \beta_2$: Adam momentum parameters;\;
$\epsilon$: Small stability constant.
}

Evaluate $\mathcal{L}(USV^\top)$\tcc*{Forward evaluate}
$G_U\gets\nabla_{U}\mathcal{L}(USV^\top);\,G_V\gets\nabla_{V}\mathcal{L}(USV^\top)$ \tcc*{Backprop}

$\left\{
\begin{array}{l}
\widehat U\gets\texttt{  basis\_augmentation}(U, G_U) \\
\widehat V\gets\texttt{  basis\_augmentation}(V, G_V)
\end{array}
\right.${\tcc*{in parallel}}

$\bar S \gets \widehat U^{\top}USV^\top\widehat V$, 
$\bar S_{\mathcal{V}} \gets \widehat U^{\top}US_{\mathcal{V}} V^\top\widehat V$, 
$\bar{S}_{\mathcal{K}}\gets \left(\widehat{U}^{\top}U\sqrt{S_{\mathcal{K}}}V^{\top}\widehat V\right)^2$\\

Evaluate $\mathcal{L}(\widehat U \bar S \widehat V^\top)$\tcc*{Forward evaluate}
$G_S\gets\nabla_{\bar S}\mathcal{L}(\widehat U \bar S \widehat V^\top)$ \tcc*{Backprop}
$\widehat S_{\mathcal{V}} \gets \beta_1 \bar S_{\mathcal{V}} + (1-\beta_1) G_S$\;
  $\widehat S_{\mathcal{K}}\gets \beta_2\bar S_{\mathcal{K}} + (1-\beta_2)\left(G_S\right)^2$\;
  \Comment{Modifications for adaptive update}
$\check{S}_{\mathcal{V}}  \gets\frac{ \widehat S_{\mathcal{V}}^{n} }{1 - \beta_1^n},\, \check{ S}_{\mathcal{K}}  \gets \frac{ \widehat S_{\mathcal{K}}^{n}}{1 - \beta_2^n}\,$ \tcc*{Bias correction}
$ \widehat S^{1} \gets \bar S - \lambda \frac{\check{ S}_{\mathcal{V}}}{\sqrt{\check{S}_{\mathcal{K}} + \epsilon}}$ \tcc*{Adaptive coefficient update}
$U,S,V,   S_{\mathcal{V}},  S_{\mathcal{K}} \gets ${\tt truncation}$(\widehat S, \widehat S_{\mathcal{V}}, \widehat S_{\mathcal{K}}, \widehat U, \widehat V; \tau  )$
\caption{Single iteration of the low-rank Adam method. \\ The functions \texttt{basis\_augmentation} and \texttt{truncation} are detailed in \ref{alg_helper} in the appendix. }\label{alg:adam_dlrt}
\end{algorithm}

A naive update of the exponentially weighted moving average of the first and second moment of the gradients with respect to the coefficients $S$, which we denote as $\mathcal{V}_S,\mathcal{K}_S\in\mathbb{R}^{2r\times 2r}$ would read as
\begin{align*}
\mathcal{V}_S^{n+1} = \beta_1 \mathcal{V}_S^{n} + (1-\beta_1)\nabla_S \mathcal{L}\qquad\text{and}\qquad
\mathcal{K}_S^{n+1} = \beta_2 \mathcal{K}_S^{n} + (1-\beta_2)(\nabla_S \mathcal{L})^2\,.
\end{align*}
The bases $U,V$ of the weight factorization could be updated as in \Cref{alg:heavy_ball_dlrt}.
However, this naive approach does not account for the fact that $\mathcal{K}_S^{n}$ and $\mathcal{K}_S^{n+1}$, respectively $\mathcal{V}_S^{n}$ and $\mathcal{V}_S^{n+1}$ belong to different bases, which are updated between optimization steps by the augmentation-truncation mechanic. 

Instead, we leverage the discussion in Section~\ref{sec:DLRT-heavy-ball}, and propose to project the previous weighted moving average to the updated bases. 
We note that through the construction of the low-rank bases, no step-size control is required to update $U$ and $V$ since the basis spans weights and moments at the old and current time steps and linear combinations in between.
To facilitate the discussion, we denote the first low-rank moment as $S_\mathcal{V}\in\mathbb{R}^{r\times r}$, analogously to the momentum term in \Cref{alg:heavy_ball_dlrt} and note that the update of $S_\mathcal{V}$ can be performed using the strategy of  \Cref{alg:heavy_ball_dlrt}
with $\bar S_{\mathcal{V}} := \widehat U^{\top} U^{0}S_{\mathcal{V}}^{0}V^{0,\top}\widehat V$ and $\bar S := \widehat U^{\top} U^{0}S^{0}V^{0,\top}\widehat V$, i.e.,
\begin{alignat*}{2}
    \widehat S_{\mathcal{V}}^{1} =\,& \beta_1\bar S_{\mathcal{V}}  +(1-\beta_1)\nabla_{\bar S} \mathcal{L}(\widehat U\bar S\widehat V^{\top})\,
\end{alignat*}
The dynamics of the second moment do not follow the gradient flow directly, but have non-linear dynamics depending on the square of the gradient. 
We remark that, the (full-rank) second moment $\hat{\mathcal{K}}$ is element-wise positive, which is important for taking the square root in the (full-rank) Adam update step  $W^{n+1} = W^n - \lambda \frac{\hat{\mathcal{V}}}{\sqrt{\hat{\mathcal{K}}} + \epsilon}$.  
Denoting the second moment of the coefficient matrix by $S_\mathcal{K}\in\mathbb{R}^{r\times r}$, we propose a modified projection 
$
\bar{S}_{\mathcal{K}} = \left(\widehat{U}^{\top}U^{0}\sqrt{S_{\mathcal{K}}^0}V^{0,\top}\widehat V\right)^2,
$
which projects the element-wise square root of the second moment of the low-rank coefficient $S$ onto the new basis and subsequently squares the elements. This ensures element-wise positivity of the second low-rank moment. The augmented low-rank moment $\bar{S}_{\mathcal{K}}$ is subsequently updated with the standard Adam update scheme
\begin{alignat*}{2}
    \widehat S_{\mathcal{K}}^{1} =\,& \beta_2\bar S_{\mathcal{K}} + (1-\beta_2)\left(\nabla_{\bar S} \mathcal{L}(\widehat U\bar S\widehat V^{\top})\right)^2.
\end{alignat*}
We apply the Adam bias correction at iteration $n$ to the low-rank moments, i.e., 
\begin{align*}
\check{S}_{\mathcal{V}}  =\frac{ \widehat S_{\mathcal{V}}^{n} }{1 - \beta_1^n}, \quad \textup{and} \quad \check{ S}_{\mathcal{K}} = \frac{ \widehat S_{\mathcal{K}}^{n}}{1 - \beta_2^n}\,,
\end{align*}
and update steps for the coefficient $S$, i.e., $
 \widehat S^{1} =\, \bar S  - \lambda \frac{\check{ S}_{\mathcal{V}}}{\sqrt{\check{S}_{\mathcal{K}} + \epsilon}}$.

The resulting method is summarized in Algorithm~\ref{alg:adam_dlrt}. We wish to remark that the extension to AdamW is straightforward: The regularization term is added to $\widehat{S}$ directly instead of being combined with the gradient for the moment updates.

\subsection*{Computational and Memory Efficiency}

We briefly analyze the computational and memory efficiency of  
1) the low-rank Heavy Ball method and  
2) the low-rank Adam method.
To update a matrix $W \in \mathbb{R}^{n \times n}$, the full-rank (baseline) Heavy Ball method requires $\mathcal{O}(3n^2)$ floats to store the weight $W$, its gradient $\nabla_W \mathcal{L}$, and the momentum $\mathcal{V}$. The computational cost is of the same order.
In contrast, the low-rank Heavy Ball method (see \Cref{alg:heavy_ball_dlrt}) requires $\mathcal{O}(2nr + r^2)$ floats to store the low-rank factorization $USV^\top$, and $\mathcal{O}(2nr + r^2)$ for the gradients $\nabla_U \mathcal{L}, \nabla_V \mathcal{L}, \nabla_S \mathcal{L}$. Since the bases $U$ and $V$ are shared between the weight and momentum, the momentum term requires only $\mathcal{O}(r^2)$ additional floats. The extra computational costs are $\mathcal{O}(nr^2)$ for orthonormalization during basis augmentation and $\mathcal{O}(r^3)$ for truncation, both negligible when $r \ll n$.
Similarly, the full-rank Adam/AdamW method requires $\mathcal{O}(4n^2)$ floats to store $W$, $\nabla_W \mathcal{L}$, and the two momentum terms $\mathcal{V}, \mathcal{K}$, with comparable compute cost. The low-rank Adam method uses $\mathcal{O}(2nr + r^2)$ for $USV^\top$ and the same for the gradients. The two momentum terms add only $\mathcal{O}(2r^2)$ due to shared bases. The computational cost is computed analogously to that of the low-rank Heavy Ball method.

\subsection*{Extension to Tensor-valued layers}
{
The proposed optimizer is directly extendable to tensor-valued neural network layers, e.g. convolutional layers, following the extension of the SGD-based DLRT method from matrices in \cite{schotthofer2022low} to tensors \cite{zangrando2023rank}.} To that end, we remark that a, e.g. 2d, convolution can be formulated as an operation on an order four tensor $\bm{W}\in\mathbb{R}^{d_i\times d_o\times s_w\times s_h}$, where $d_i$ and $d_o$ are the channels of the input and output data, and $s_w, s_h$ is the width and height of the sliding window of the convolution acting on the spatial dimensions $s,h$ of the input image $x\in\mathbb{R}^{s\times h\times d_i}$.
Consider a low-rank Tucker factorization of the weight, i.e. $\bm{W}=\bm{C}\times_{i=1}^4 U_i$, with core tensor $\bm{C}\in\mathbb{R}^{r_1\times r_2\times r_3\times r_4}$ and $U_i\in\mathbb{R}^{n_i\times r_i}$ with $n_i\in\{d_i, d_o, s_w,s_h\}$. 
Using the extension of the DLRT method to Tucker tensors \cite{NEURIPS2024_ea48cb23} and applying \Cref{alg:heavy_ball_dlrt}, respectively \Cref{alg:adam_dlrt} to the tensor update yields the desired method.

\section{Numerical Results}\label{sec:num_res}
\begin{wraptable}{r}{0.6\textwidth}
\centering
\vspace{-1.5cm}
\caption{UCM, Cifar10 and Cifar100 benchmark; Low-rank compressed transfer learning. Accuracy means and std.~devs.~of 10 stochastic trainings using AdamW. The LoRA ranks are set up to match the compression rate of the results of \Cref{alg:adam_dlrt}.  \Cref{alg:adam_dlrt} achieves higher accuracy at higher compression rates across all benchmarks, compared to DLRT\cite{schotthofer2022low} w/o momentum term projection and LoRA-based pretraining w/ momentum terms. }
\label{tab_lr_pretrain_summary}
\resizebox{0.6\textwidth}{!}{
\begin{tabular}{@{}llcccccc@{}}
\toprule
  & &  \multicolumn{2}{c}{\textbf{UCM Data}} & \multicolumn{2}{c}{\textbf{Cifar10 Data}} & \multicolumn{2}{c}{\textbf{Cifar100 Data}} \\
 \cmidrule(lr){3-4} \cmidrule(lr){5-6} \cmidrule(lr){7-8}
 &  & \textbf{Acc [\%]} & \textbf{c.r. [\%]}& \textbf{Acc [\%]} & \textbf{c.r. [\%]} & \textbf{Acc [\%]} & \textbf{c.r. [\%]} \\
 \midrule
\multirow{4}{*}{\rotatebox[origin=c]{90}{{VGG16}}}
& Baseline & 94.40$\pm$0.72 & 0.0& \textbf{89.82$\pm$0.45} & 0.0 & \textbf{65.21$\pm$0.37} & 0.0 \\
& \Cref{alg:adam_dlrt} &  \textbf{94.61$\pm$0.35} &95.84 & 89.49$\pm$0.58 &95.30  & 64.58$\pm$0.46 &95.54\\
& DLRT w/o proj. & 89.32$\pm$0.93 & 93.56 &85.01$\pm$0.28 & 94.84  &60.48$\pm$0.27 & 98.71\\
& LoRA pretrain & 90.64$\pm$2.27 & 93.57 &88.43$\pm$ 0.23 & 94.80& 61.63$\pm$0.46 &96.58  \\
\midrule
\multirow{4}{*}{\rotatebox[origin=c]{90}{{VGG11}}}
& Baseline & \textbf{94.23$\pm$0.71} & 0.0 & \textbf{88.34$\pm$0.49} &0.0 & \textbf{63.13$\pm$0.41} & 0.0\\
& \Cref{alg:adam_dlrt} &93.70$\pm$0.71 & 94.89 &88.13$\pm$0.56 & 95.13 &60.84$\pm$0.40 & 95.08 \\
& DLRT w/o proj. & 88.23$\pm$0.90  & 90.35 & 81.98$\pm$0.25 & 97.08 & 61.59$\pm$0.25 & 95.99 \\
& LoRA pretrain & 90.14$\pm$2.56 &94.72 &  86.63$\pm$0.29 & 94.57&59.54$\pm$0.40 & 94.78 \\
\midrule
\multirow{4}{*}{\rotatebox[origin=c]{90}{{ViT-B.16}}}
& Baseline & \textbf{96.72$\pm$0.36}  &0.0& \textbf{95.42$\pm$0.35} & 0.0 & \textbf{90.34$\pm$0.44} & 0.0\\
& \Cref{alg:adam_dlrt} & 96.38$\pm$0.60 & 86.7  &95.39$\pm$0.41   &83.42   &88.48$\pm$0.53 & 75.38\\
& DLRT w/o proj. & 78.94$\pm$0.50 &84.91&  91.95$\pm$0.50 &84.95&75.09$\pm$0.53 & 75.83\\
& LoRA pretrain &  86.54$\pm$2.91 & 84.94 & 94.10$\pm$0.56 &80.78 &76.76$\pm$0.53 & 74.86\\
\bottomrule
\end{tabular}
}
\vspace{-0.5cm}
\end{wraptable}

In the following, we showcase numerical experiments for low-rank transfer learning (\Cref{subsec_transferlearning}), finetuning (\Cref{subsec_peft}), and pre-training tasks (\Cref{subsec_pretrain}). Details on training, data and additional experiments are listed in \Cref{app_experiments}.

Our primary baseline is \emph{full-rank training}, where all model parameters are updated without any structural constraints. We compare this against 1) the low-rank finetuning strategy introduced in \Cref{alg:adam_dlrt} 2) naive application of DLRT without the projection of the momentum terms, and 3) simultaneous direct application of the Adam optimizer on the low-rank factors $U,S,V$, as typically done in 
{LoRA}~\cite{hu2021lora}, a low-rank adaptation technique originally designed for parameter-efficient finetuning of large transformers. 
For LoRA-based experiments, we calibrate the per-layer rank hyperparameters to match the overall \emph{compression ratio} achieved by \Cref{alg:adam_dlrt}, ensuring a fair comparison at fixed parameter budgets. The compression ratio is defined as:
$\textup{c.r.} = \left(1 - \frac{\#\textup{params low-rank model}}{\#\textup{params baseline model}} \right) \times 100.$

\subsection{Low-Rank Compressed Transfer Learning}\label{subsec_transferlearning}
\paragraph{VGG16, VGG11, and ViT-B.16 on UCM/CIFAR-10/CIFAR-100}
We evaluate performance across three network architectures, VGG16 and VGG11, and the ViT-B.16 Vision Transformer on UCM/CIFAR-10/CIFAR-100.
The convolutional VGG11 and VGG16 networks are selected to validate the performance of \Cref{alg:adam_dlrt} on tensor-valued layers, here given by the convolutional layers. We use the low-rank Tucker tensor format to compress and train the convolutions; for details, we refer to \cite{schotthöfer2025dynamicallowrankcompressionneural, NEURIPS2024_ea48cb23}.

\begin{table}[b]
\centering
\caption{DeBERTaV3-base finetuning on GLUE. We compare with full finetuning (Full FT), Houlsby adapter \citep{pmlr-v97-houlsby19a} (HAdapter),  Pfeiffer adapter \citep{pfeiffer2021adapterfusionnondestructivetaskcomposition} (PAdapter), LoRA \citep{hu2021lora}, AdaLoRA \citep{zhang2023adalora}, GeoLoRA\citep{schotthöfer2024GeoLoRAgeometricintegrationparameter}, DoRA \citep{Mao2024DoRAEP}, LoRA+\citep{hayou2024lora}, and Bitfit\citep{zaken2022bitfitsimpleparameterefficientfinetuning}. We report target metrics and computational performance (higher is better) for the median of 5 runs using different random seeds. Best results per dataset are shown
in bold. Results for BitFit, HAdapter, and PAdapter were taken from \citep{zhang2023adalora}. "AdaLoRa matched" has the rank budget adapted to approximately match the final parameter count of \Cref{alg:adam_dlrt}.}
\label{tab_results}
\resizebox{\textwidth}{!}
{%
\begin{tabular}{lccccccccc}
\toprule
\textbf{Method} (\textcolor{blue}{\# Params})& \textbf{ SST-2}  & \textbf{CoLA} & \textbf{QQP}  & \textbf{QNLI}  & \textbf{RTE}  & \textbf{MRPC}  & \textbf{STS-B} & \textbf{Mean}  \\
  &(Acc) &(Mcc) &(F1)&(Acc)&(Acc)&(Acc)&(Corr)\\
\midrule
Full FT  (\textcolor{blue}{184M})  & 95.63 & 69.19 & 89.80 & 94.03 & 83.75 & 89.46 & 91.60  &87.63 \\
\midrule
BitFit     (\textcolor{blue}{0.1M})  & 94.84 & 66.96 & 84.95 & 92.24 & 78.70 & 87.75 & 91.35 &85.25 \\
HAdapter  (\textcolor{blue}{1.22M})   & 95.53 & 68.64 & 89.27 & 94.11 & 84.48 & 89.95 & 91.48 &87.63 \\
PAdapter   (\textcolor{blue}{1.18M})   & 95.61 & 68.77 & 89.40 & \textbf{94.29} & 85.20 & 89.46 & 91.54 &87.75 \\
\midrule
LoRA r=8 (\textcolor{blue}{1.33M}) & 95.29 & 68.57 & 90.61 & 93.91& 85.50 & 89.75 & 89.10 &87.53\\
 {LoRA+ r=8} (\textcolor{blue}{1.33M})   &  {95.37} &  {{69.22}} &  {\textbf{90.82}}  & {93.96} &  {85.50} & {89.55} &  {88.07} &  {87.49}\\ 
 {DoRA r=8} (\textcolor{blue}{1.33M})  &    {94.30} &  {68.50} &  {90.71} &  {94.31} &  {85.05} &  {89.32} &  {91.38} & {87.65}\\
AdaLoRA r$_{f}=8$ (\textcolor{blue}{1.27M})  & 95.64  & 68.76 & {90.65} & 94.11 & {86.00} & 89.44 & 91.41 &88.00\\
AdaLoRA, matched  & 95.64  (\textcolor{blue}{1.27M})  & 68.59  (\textcolor{blue}{1.07M}) & {90.48} (\textcolor{blue}{0.72M})  & 93.93 (\textcolor{blue}{0.72M})& 85.92 (\textcolor{blue}{1.16M}) & 88.21 (\textcolor{blue}{0.74M})& 90.91(\textcolor{blue}{0.74M}) &87.66 (\textcolor{blue}{0.91M})\\

GeoLoRA  &{95.98} (\textcolor{blue}{1.17M})&69.03 (\textcolor{blue}{0.98M})&90.53 (\textcolor{blue}{0.69M}) &94.23 (\textcolor{blue}{0.70M})&85.93 (\textcolor{blue}{1.19M})&{90.10} (\textcolor{blue}{0.75M}) & \textbf{91.58} (\textcolor{blue}{0.71M}) &{88.19} (\textcolor{blue}{0.88M}) \\
\midrule
\Cref{alg:adam_dlrt} &    \textbf{96.02} (\textcolor{blue}{1.11M})&  \textbf{69.58} (\textcolor{blue}{1.01M}) &{90.62} (\textcolor{blue}{0.76M}) &  {94.02}  (\textcolor{blue}{0.70M})&\textbf{88.67} (\textcolor{blue}{1.19M})&\textbf{90.84} (\textcolor{blue}{0.76M}) & {91.51} (\textcolor{blue}{0.73M}) & \textbf{88.75}(\textcolor{blue}{0.89M})\\

\bottomrule
\end{tabular}
}
\end{table}

We initialize all convolutional networks with PyTorch Imagenet1K weights and ViT-B.16 with Huggingface Imagenet21K weights. Stochasticity during training stems from randomized mini-batching. For each experiment, we report the mean performance across 10 independent training runs with different random seeds. We observe in \Cref{tab_lr_pretrain_summary} that \Cref{alg:adam_dlrt} matches the validation accuracy of the baseline network in most test-cases, and surpasses the baseline in e.g. UCM/VGG16 while achieving compression rates of up to $95\%$. We point out that a naive implementation of  DLRT without the proposed adaptation of the optimizer states causes performance drops of 5 to 13\%. LoRA also struggles to achieve high accuracy at the prescribed compression rates.

\subsection{Low Rank Adaptation for Parameter Efficient Finetuning (PEFT)}\label{subsec_peft}
\paragraph{DeBERTaV3-base on GLUE}
We fine-tune the 183M parameter DeBERTaV3-base transformer model~\cite{he2023debertav3improvingdebertausing} on the GLUE benchmark suite~\cite{wang2019gluemultitaskbenchmarkanalysis}. The corresponding results are summarized in Table~\ref{tab_results}. Overall, \Cref{alg:adam_dlrt} consistently outperforms competing methods, especially other rank adaptive methods as GeoLoRA \cite{schotthöfer2024GeoLoRAgeometricintegrationparameter} and AdaLoRA \cite{zhang2023adalora}, on most tasks, achieving stronger validation metrics.  The required number of trainable parameters is substantially lower than the compared fixed-rank methods. The average score is higher than the reference methods, and the average parameter count for the finetuning tasks is lower than the next best method, which is GeoLoRA.

\begin{wraptable}{r}{0.5\textwidth}
\vspace{-0.6cm}
\centering
\caption{Llama2 7b-chat-hf \cite{touvron2023llama2} finetuning on reasoning datasets. We compare with LoRA \citep{hu2021lora} and report the best accuracy and the wall-time. The wall-time is reported for three epochs with batch size 12 and maximal sequence length of 640 tokens on a single NVIDIA H100.}
\label{tab:Llama}
\resizebox{0.5\textwidth}{!}
{%
\begin{tabular}{@{}lllccccccccc@{}}
\toprule
  \textbf{Method} & \multicolumn{3}{c}{BoolQ} & \multicolumn{3}{c}{PIQA} \\
 \cmidrule(lr){2-4} \cmidrule(lr){5-7} \cmidrule(lr){8-10}
 &  \textbf{c.r. [\%]} & \textbf{Acc [\%]} & \textbf{Wall-Time} & \textbf{c.r. [\%]} & \textbf{Acc [\%]} & \textbf{Wall-Time} \\
 \midrule
 \Cref{alg:adam_dlrt}   & 99.73\%   & \textbf{84.09 \%} & 186min & 99.75\% & \textbf{76.77\%} & 228min \\
 LoRA (r=6)             & 99.82\%   & 62.17 \% & 173min  & 99.82\% & 52.18\% & 225min &\\
LoRA  (r=10)            & 99.70\%   & 62.17 \% & 184min  & 99.70\%& 50.43\% & 225min &\\
\bottomrule
\end{tabular}
}
\vspace{-0.5cm}
\end{wraptable}
\paragraph{Llama2 7b-chat-hf on BoolQ and PIQA}
We compare \Cref{alg:adam_dlrt} with LoRA on Llama-2-7b-chat-hf \cite{touvron2023llama2} across reasoning benchmarks, including BoolQ \cite{clark2019boolq} and PIQA \cite{Bisk2020PIQA}, as reported in Table~\ref{tab:Llama}. Inputs consist of either a passage–question pair or a standalone question with multiple-choice answers, and evaluation is based on answer accuracy. We also report wall-clock time on a single NVIDIA H100 GPU, showing negligible runtime overhead of \Cref{alg:adam_dlrt} over LoRA. \Cref{alg:adam_dlrt} outperforms LoRA configurations with matching initial rank and with rank chosen to approximately match the final parameter count.

\vspace{-0.25cm}
\subsection{Low Rank Pretraining}\label{subsec_pretrain}
\vspace{-0.1cm}

\paragraph{GPT2 on OpenWebText}
\begin{wraptable}{r}{0.4\textwidth}
\vspace{-1.3cm}
\centering
\caption{Pretraining GPT-2 \cite{radford2019language} reproduction (124M) from scratch on OpenWebText \cite{Gokaslan2019OpenWeb} for 15,000 iterations.}
\label{tab:GPT2}
\resizebox{0.4\textwidth}{!}
{%
\begin{tabular}{lccc}
\toprule
\textbf{Method}         & \textbf{c.r. [\%]} & \textbf{validation loss [\%]} \\
\midrule
Baseline                & 0     & {3.2313} \\
\midrule
\Cref{alg:adam_dlrt}    &  60.61\%     & \textbf{3.4642} \\
LoRA Pretrain           &  60.79\%     & 7.0242 \\
\bottomrule
\end{tabular}
}
\vspace{-0.5cm}
\end{wraptable}
We pretrain Karpathy’s reproduction\footnote{\url{https://github.com/karpathy/nanoGPT}} of the 124M-parameter GPT-2 model \cite{radford2019language} from scratch on the OpenWebText dataset \cite{Gokaslan2019OpenWeb} using next-word prediction. As seen in \Cref{tab:GPT2} and \Cref{fig_GPT2} , our method significantly outperforms LoRA pretraining (best validation loss $3.4642$ vs.\ $4.8141$), while incurring only a moderate increase relative to the full-rank baseline ($3.4642$ vs.\ $3.2313$). Algorithm~2 achieves a compression rate of $60.61\%$, compared to $60.79\%$ for LoRA-Pretrain\footnote{We note that approximately 31\% of the total parameters are in the Head.}.  Thus, our approach enables substantial compression of GPT-2 (with the potential for reduced inference time), whereas LoRA yields a significant degradation in validation loss.

\vspace{-0.3cm}
\section{Conclusion}
\vspace{-0.1cm}

We introduced a principled and provably robust framework for momentum-based low-rank optimization that is both rank-adaptive and jointly compresses weights, gradients, and optimizer states. Our analysis reveals that the proposed method is resilient to the conditioning of the training problem, while maintaining fidelity to full-rank momentum trajectories. Through extensive experiments on pretraining, transfer learning, and finetuning, we demonstrate that our approach consistently achieves stronger generalization performance under tight parameter budgets, outperforming existing low-rank techniques. These results position our method as a strong foundation for efficient deep learning, offering a scalable and theoretically grounded alternative to full-rank training across diverse regimes.
The accomplished faster convergence and higher compression of the optimizer states during training enable broader applications of machine learning on resource-constrained devices. Furthermore, the method is computationally efficient and scalable. These achievements also enhance computational and memory efficiency, positively impacting society. 

\section*{Funding Acknowledgements}{

This material is based upon work supported by the Laboratory Directed Research and Development Program of Oak Ridge National Laboratory (ORNL), managed by UT-Battelle, LLC for the U.S.~Department of Energy under Contract No. De-AC05-00OR22725.

This manuscript has been authored by UT-Battelle, LLC under Contract No. DE-AC05-00OR22725 with the U.S.~Department of Energy. The United States Government retains and the publisher, by accepting the article for publication, acknowledges that the United States Government retains a non-exclusive, paid-up, irrevocable, world-wide license to publish or reproduce the published form of this manuscript, or allow others to do so, for United States Government purposes. The Department of Energy will provide public access to these results of federally sponsored research in accordance with the DOE Public Access Plan(\url{http://energy.gov/downloads/doe-public-access-plan}).

This project has received funding from the European Regional Development Fund (grants timingMatters and IntelAlgen) under the European Union’s Horizon Europe Research and Innovation Program, from the German Research Foundation DFG within GRK 2297 ‘Mathematical Complexity Reduction’, and from the German Federal Joint Committee (Grant 01VSF23017), which we gratefully acknowledge.

\newpage
\bibliographystyle{abbrv}
\bibliography{main.bib}


\newpage

\section{Notation }

\begin{table}[h]
\centering
\caption{Summary of notation used throughout the paper.}
\label{tab:notation}
\begin{tabular}{ll}
\toprule
\textbf{Notation} & \textbf{Definition} \\
\midrule
\multicolumn{2}{l}{\textit{\textbf{Model \& Training}}} \\
$W_l \in \mathbb{R}^{n_l \times n_{l-1}}$ & The weight matrix for layer $l$. \\

$\mathcal{L}$ & The loss function. \\
$\nabla_W \mathcal{L}$ & The gradient of the loss with respect to the full matrix $W$. \\
$\lambda$ & The learning rate of stochastic gradient descent. \\
\addlinespace
\multicolumn{2}{l}{\textit{\textbf{Low-Rank Factorization}}} \\
$\mathcal{M}_r$ & The manifold of matrices of rank $r$. \\
$P(W)Z$ & Orthogonal projection of a matrix $Z$ onto the tangent space of $\mathcal{M}_r$ at $W$. \\

$W = USV^\top$ & Low-rank decomposition of $W$, where $U, V$ are orthonormal. \\
$U \in \mathbb{R}^{n_l \times r}, V \in \mathbb{R}^{n_{l-1} \times r}$ & Orthonormal basis matrices for the column and row spaces. \\
$S \in \mathbb{R}^{r \times r}$ & The core tensor or coefficient matrix. \\
$\nabla_U \mathcal{L}, \nabla_S \mathcal{L}, \nabla_V \mathcal{L}$ & Gradients with respect to the low-rank factors $U, S, V$. \\
\addlinespace
\multicolumn{2}{l}{\textit{\textbf{Momentum Terms (Heavy Ball)}}} \\
$\mathcal{V}$ & The momentum term (a full-rank matrix). \\
$U_{\mathcal{V}} S_{\mathcal{V}} V_{\mathcal{V}}^\top$ & Low-rank decomposition of the momentum term. \\
$S_{\mathcal{V}} \in \mathbb{R}^{r \times r}$ & The coefficient matrix for the momentum term. \\
$\gamma$ & The momentum decay parameter. \\
\addlinespace
\multicolumn{2}{l}{\textit{\textbf{Momentum Terms (Adam)}}} \\
$\mathcal{V}$ & The momentum term (a full-rank matrix). \\
$U_{\mathcal{V}} S_{\mathcal{V}} V_{\mathcal{V}}^\top$ & Low-rank decomposition of 1st moment term (moving average of gradients). \\
$S_\mathcal{V} \in \mathbb{R}^{r \times r}$ & Coefficient matrix for the 1st moment (moving average of gradients). \\
$\mathcal{K}$ & The momentum term (a full-rank matrix). \\
$U_{\mathcal{K}} S_{\mathcal{K}} V_{\mathcal{K}}^\top$ & Low-rank decomposition of the 2nd moment. \\
$S_\mathcal{K} \in \mathbb{R}^{r \times r}$ & Coefficient matrix for the 2nd moment (moving average of squared gradients). \\
$\beta_1, \beta_2$ & Exponential decay rates for the moment estimates. \\
\addlinespace
\multicolumn{2}{l}{\textit{\textbf{Algorithm-Specific}}} \\
$\hat{U}, \hat{V}$ & Augmented bases after incorporating gradient information. \\
$\overline{S}$ & Coefficient matrix $S$ projected onto the augmented bases $\hat{U}, \hat{V}$. \\
$\overline{S}_\mathcal{V}, \overline{S}_\mathcal{K}$ & Momentum coefficients projected onto the new augmented bases. \\
$\hat{S}$ & The final updated coefficient matrix within a single optimization step. \\
$\tau, \vartheta$ & Relative and absolute truncation tolerance parameter for rank adaptation. \\
\bottomrule
\end{tabular}
\end{table}

\section{Algorithms}
We list the helper functions of \Cref{alg:heavy_ball_dlrt} and \Cref{alg:heavy_ball_dlrt} in \Cref{alg_helper}. The baseline (full-rank) Heavyball method is listed in \cref{alg_heavyball_original} and the  (full-rank)  Adam method is listed in \Cref{alg_adam}. {The naive application of the DLRT method \cite{NEURIPS2024_ea48cb23} for Adam is listed in \Cref{alg:adam_dlrt_naive}. Note that the coefficient matrices of the momentum terms are not updated after augmentation and truncation. Analogously, the naive implementation of DLRT with momentum omits updating the momentum coefficient matrix. 
The implementation of LoRA with momentum or Adam simply applies \Cref{alg_adam} to each of the LoRA factors without regard for the underlying manifold representation. 
} 
\begin{algorithm}[h]

\DontPrintSemicolon
\SetAlgoLined

\vspace{.2em}
\SetKwProg{Def}{def}{:}{}
\Def{   {\tt  basis\_augmentation}($B$: old basis, $G_B$: basis dynamics)}{
$\widehat B \gets  \texttt{ortho}([G_B \mid B])$ \tcc*{orthonormalization, e.g. Gram-Schmidt}
return  $\widehat B$
}

\vspace{.2em}
\SetKwProg{Def}{def}{:}{}
\Def{   {\tt  truncation}($\augS$: augmented coefficient, $\augS_\mathcal{V}$: augmented momentum, $\augU$: augmented basis, $\augV$: augmented co-basis )}{
$P_{r_1}, \Sigma_{r_1}, Q_{r_1} \gets$ truncated \texttt{svd}$(\augS)$ with threshold  $\vartheta$ to new rank $r_1$\;
$U\gets   \augU P_{r_1}$;
$V\gets   \augV Q_{r_1}$ \tcc*{Basis update}
$S\gets \Sigma_{r_1}; S_\mathcal{V} \gets U^{\top}\augU \augS_\mathcal{V} \augV^{\top}V$ \tcc*{Coefficient update}
return  $U,S,V,S_{\mathcal{V}}$\;
}

\vspace{.2em}
\SetKwProg{Def}{def}{:}{}
\Def{   {\tt  truncation}($\augS$: augmented coefficient, $\augS_\mathcal{V}$: augmented momentum,  $\augS_\mathcal{K}$: augmented 2nd momentum, $\augU$: augmented basis, $\augV$: augmented co-basis )}{
$P_{r_1}, \Sigma_{r_1}, Q_{r_1} \gets$ truncated \texttt{svd}$(\augS)$ with threshold  $\vartheta$ to new rank $r_1$\;
$U\gets   \augU P_{r_1}$;
$V\gets   \augV Q_{r_1}$ \tcc*{Basis update}
$S\gets \Sigma_{r_1}; S_\mathcal{V} \gets U^{\top}\augU \augS_\mathcal{V} \augV^{\top}V; \augS_\mathcal{K}\gets \left( U^{\top}\augU  \sqrt{S_\mathcal{K} }\augV^{\top}V\right)^2$ \tcc*{Coefficient update}
return  $U,S,V,S_\mathcal{V}, S_\mathcal{K}$\;
}

\vspace{.2em}
\SetKwProg{Def}{def}{:}{}
\Def{   {\tt  truncation\_naive}($\augS$: augmented coefficient, $\augU$: augmented basis, $\augV$: augmented co-basis )}{
$P_{r_1}, \Sigma_{r_1}, Q_{r_1} \gets$ truncated \texttt{svd}$(\augS)$ with threshold  $\vartheta$ to new rank $r_1$\;
$U\gets   \augU P_{r_1}$;
$V\gets   \augV Q_{r_1}$ \tcc*{Basis update}
$S\gets \Sigma_{r_1}$;\tcc*{Coefficient update}
}\caption{The functions \texttt{basis\_augmentation}, and \texttt{truncation} of the used algorithms. }\label{alg_helper}
\end{algorithm}

\begin{algorithm}[h]
\DontPrintSemicolon
\SetAlgoLined
\SetKwInOut{Input}{Input}
\SetKwComment{Comment}{$\triangleright$\ }{}

\Input{ $W \in \mathbb{R}^{n\times n}$: Weight matrix;\;
$\mathcal{V}\in \mathbb{R}^{n\times n}$: Initial 1st moment;\;
$\mathcal{K}\in \mathbb{R}^{n\times n}$: Initial 2nd moment;\;
$\lambda$: learning rate;\;
$\beta_1, \beta_2$: Adam momentum parameters;\;
$\epsilon>0$: Small stability constant.
}

Evaluate $\mathcal{L}(W)$ \\
$g \gets \nabla_{W} \mathcal{L}(W)$ \tcc*{Compute gradient}

$\mathcal{V}\gets \beta_1 \mathcal{V} + (1 - \beta_1) g$ \tcc*{1st moment estimate}
$\mathcal{K} \gets \beta_2 \mathcal{K} + (1 - \beta_2) g^2$ \tcc*{2nd moment estimate (element-wise square)}

$\hat{\mathcal{V}} \gets \frac{\mathcal{V}}{1 - \beta_1^t}, \quad \hat{\mathcal{K}} \gets \frac{\mathcal{K}}{1 - \beta_2^t}$ \tcc*{Bias correction}

$W \gets W - \lambda \frac{\hat{\mathcal{V}}}{\sqrt{\hat{\mathcal{K}}} + \epsilon}$ \tcc*{Parameter update}

\caption{Single iteration of the (full-rank version of) Adam.  }\label{alg_adam}
\end{algorithm}

\begin{algorithm}[t]
\DontPrintSemicolon
\SetAlgoLined
\SetKwInOut{Input}{Input}
\SetKwComment{Comment}{$\triangleright$\ }{}

\Input{Initial orthonormal bases $U,V\in\mathbb{R}^{n\times r}$ and coefficients $S, S_{\mathcal{V}}, S_{\mathcal{K}}\in\mathbb{R}^{r\times r}$;\;
$\tau$: singular value threshold for rank truncation;\;
$\lambda$: learning rate;\;
$\beta_1, \beta_2$: Adam momentum parameters;\;
$\epsilon$: Small stability constant.
}

Evaluate $\mathcal{L}(USV^\top)$\tcc*{Forward evaluate}
$G_U\gets\nabla_{U}\mathcal{L}(USV^\top);\,G_V\gets\nabla_{V}\mathcal{L}(USV^\top)$ \tcc*{Backprop}

$\left\{
\begin{array}{l}
\widehat U\gets\texttt{  basis\_augmentation}(U, G_U) \\
\widehat V\gets\texttt{  basis\_augmentation}(V, G_V)
\end{array}
\right.${\tcc*{in parallel}}

$\bar S \gets \widehat U^{\top}USV^\top\widehat V$\\

Evaluate $\mathcal{L}(\widehat U \bar S \widehat V^\top)$\tcc*{Forward evaluate}
$G_S\gets\nabla_{\bar S}\mathcal{L}(\widehat U \bar S \widehat V^\top)$ \tcc*{Backprop}
$\widehat S_{\mathcal{V}} \gets \beta_1 \bar S_{\mathcal{V}} + (1-\beta_1) G_S$\;
  $\widehat S_{\mathcal{K}}\gets \beta_2\bar S_{\mathcal{K}} + (1-\beta_2)\left(G_S\right)^2$\;
  \Comment{Modifications for adaptive update}
$\check{S}_{\mathcal{V}}  \gets\frac{ \widehat S_{\mathcal{V}}^{n} }{1 - \beta_1^n},\, \check{ S}_{\mathcal{K}}  \gets \frac{ \widehat S_{\mathcal{K}}^{n}}{1 - \beta_2^n}\,$ \tcc*{Bias correction}
$ \widehat S^{1} \gets \bar S - \lambda \frac{\check{ S}_{\mathcal{V}}}{\sqrt{\check{S}_{\mathcal{K}} + \epsilon}}$ \tcc*{Adaptive coefficient update}
$U,S,V,   S_{\mathcal{V}},  S_{\mathcal{K}} \gets ${\tt truncation\_naive}$(\widehat S, \widehat U, \widehat V; \tau  )$
\caption{Single iteration of the naive low-rank Adam method. \\ The functions \texttt{basis\_augmentation}, and \texttt{truncation\_naive} are detailed in \ref{alg_helper} in the appendix. }\label{alg:adam_dlrt_naive}
\end{algorithm}

\begin{algorithm}[h]
\DontPrintSemicolon
\SetAlgoLined
\SetKwInOut{Input}{Input}
\SetKwComment{Comment}{$\triangleright$\ }{}

\Input{Initial parameter vector $W \in \mathbb{R}^{n\times n}$;\;
$\mathcal{V}$: Initial velocity (momentum term);\;
Gradient $g = \nabla_{W} \mathcal{L}(W)$;\;
$\lambda$: learning rate;\;
$\gamma$: momentum coefficient.
}

Evaluate $\mathcal{L}(W)$ \\
$g \gets \nabla_{W} \mathcal{L}(W)$ \tcc*{Compute gradient}

$\mathcal{V} \gets (1-\gamma) \mathcal{V} - \lambda g$ \tcc*{Update velocity}

$W \gets W + \lambda\mathcal{V}$ \tcc*{Parameter update}

\caption{Single iteration of the (full-rank version of) the Heavy-Ball SGD method.  }\label{alg_heavyball_original}
\end{algorithm}

\section{Numerical Analysis}\label{app:robusterror}

\begin{theorem}[Convergence]\label{th:lr_conv}
Let $W(t)$ be the solution of {Eq.} \eqref{eq:gradflowopt} and let $\mathcal{L}$ be bounded from below. Then, $W(t)$ converges to a $W^{\star}$ which fulfills the low-rank optimality condition
\begin{align}
    P(W^{\star})\nabla_W\mathcal{L}(W^{\star}) = 0\,.
\end{align}
\end{theorem}
\begin{proof}
Let us define the energy as
\begin{align*}
    E(t) := \mathcal{L}(W(t)) + \frac12 \Vert \mathcal{V}(t) \Vert^2\,.
\end{align*}
The time derivative is given by
\begin{align*}
    \dot E(t) :=\,& \langle \nabla_W\mathcal{L}(W(t)), \dot W(t)\rangle + \langle \mathcal{V}(t), \dot{\mathcal{V}}(t)\rangle\\
    =\,& \langle \nabla_W\mathcal{L}(W(t)), P(W(t))\mathcal{V}(t)\rangle + \langle \mathcal{V}(t), -\gamma \mathcal{V}(t) - P(W(t))\nabla_W\mathcal{L}(W(t))\rangle\,.
\end{align*}
Since $P$ is self-adjoint this directly gives
\begin{align*}
    \dot E(t) =\,& \langle P(W(t))\nabla_W\mathcal{L}(W(t)), \mathcal{V}(t)\rangle + \langle \mathcal{V}(t), -\gamma \mathcal{V}(t) - P(W(t))\nabla_W\mathcal{L}(W(t))\rangle \\
    =\,& -\gamma \Vert \mathcal{V}(t) \Vert^2\,.
\end{align*}
Hence, if $\mathcal{L}$ is bounded from below, this means that $\lim_{t\rightarrow \infty} E(t) = E_{\infty}$ with $E_{\infty}$ finite and
\begin{align*}
    E_{\infty} = E(0) - \gamma \int_{0}^{\infty}  \Vert \mathcal{V}(t) \Vert^2\,dt\,.
\end{align*}
This implies that $\lim_{t\rightarrow \infty} \mathcal{V}(t) = 0$ and thus $\lim_{t\rightarrow \infty} \dot W(t) = \lim_{t\rightarrow \infty} P(W(t))\mathcal{V}(t) = 0$. Hence, since $\mathcal{V}(t),W(t)$ converge to a steady state and $\lim_{t\rightarrow \infty} \mathcal{V}(t) = 0$, the evolution equation for $\mathcal{V}$ gives $P(W(t))\nabla_W\mathcal{L}(W(t)) = 0$ as $t\rightarrow \infty$.
\end{proof}

We can obtain a similar, but not equivalent, result when solving a low-rank gradient flow of the form \eqref{eq:grad_flow_dlr} instead:

\begin{theorem}[Convergence of low-rank factors]\label{th:dlra_conv}
The low-rank gradient flow
\begin{subequations}\label{eq:grad_flow_dlr_app}
\begin{alignat}{2}
    \dot{U}_{\mathcal{V}}  =\,& - (I-U_{\mathcal{V}}U_{\mathcal{V}}^{\top})\nabla_W\mathcal{L}V_{\mathcal{V}}S_{\mathcal{V}}^{-1}\,, \\
    \dot{V}_{\mathcal{V}}  =\,& - (I-V_{\mathcal{V}}V_{\mathcal{V}}^{\top})\nabla_W\mathcal{L}^{\top}U_{\mathcal{V}}S_{\mathcal{V}}^{-\top}\,, \\
    \dot{S}_{\mathcal{V}}  =\,& -\gamma S_{\mathcal{V}} - U_{\mathcal{V}}^{\top}\nabla_W\mathcal{L}V_{\mathcal{V}} \,\,,
\end{alignat}
\end{subequations}
fulfills
\begin{align*}
    \dot{\mathcal{V}} = -\gamma \mathcal{V} - P(\mathcal{V})\nabla_W\mathcal{L}\,.
\end{align*}
\end{theorem}
\begin{proof}
    By the product rule we have
    \begin{align}\label{eq:dotWv}
        \dot{\mathcal{V}} =\,& \dot U_{\mathcal{V}} S_{\mathcal{V}}V_{\mathcal{V}}^{\top} + U_{\mathcal{V}} \dot S_{\mathcal{V}}V_{\mathcal{V}}^{\top} + U_{\mathcal{V}} S_{\mathcal{V}} \dot V_{\mathcal{V}}^{\top}\nonumber\\
        =\,& -\gamma \mathcal{V} - (I-U_{\mathcal{V}}U_{\mathcal{V}}^{\top})\nabla_W\mathcal{L}V_{\mathcal{V}}V_{\mathcal{V}}^{\top} - U_{\mathcal{V}} U_{\mathcal{V}}^{\top}\nabla_W\mathcal{L}V_{\mathcal{V}}V_{\mathcal{V}}^{\top} - U_{\mathcal{V}} U_{\mathcal{V}}^{\top}\nabla_W\mathcal{L}(I-V_{\mathcal{V}}V_{\mathcal{V}}^{\top})\nonumber\\
        =\,& -\gamma \mathcal{V}-P(\mathcal{V})\nabla_W\mathcal{L} \,.
    \end{align}
\end{proof}

\begin{theorem}[Error-bound]\label{th:robust_error_bound}
For an integer $k$, let $t=k\lambda$. Let $W(t)$ be the solution of {Eq.} \eqref{eq:gradflowopt}, and let $W^r_t$, $\mathcal{V}_t$ be the factorized low-rank solution after $k$ steps with Algorithm~\ref{alg:heavy_ball_dlrt}.
    Assume that for any $Z\in\mathcal{M}_r$ in a neighborhood of $W_t^r$, we have $\Vert(I-  P(Z))\nabla\mathcal{L}(Z)\Vert<\varepsilon$ and $\Vert \widehat U_t\widehat U_t^{\top}\mathcal{V}_t\widehat V_t\widehat V_t^{\top} - U_tU_t^{\top}\mathcal{V}_tV_tV_t^{\top}\Vert\leq\widehat \vartheta$, where $\Vert\cdot\Vert$ denotes the Frobenius norm. Moreover, assume that the gradient is bounded and Lipschitz continuous. Then, 
\begin{equation}\label{eq:approx}
   \Vert{W(t)-W^r_t}\Vert\leq  c_{1}\varepsilon + c_{2}\lambda +c_{3}\vartheta/\lambda +c_{4}\widehat \vartheta/\lambda\,,
\end{equation}
where the constants $c_{1}$, $c_{2}$,  $c_{3}$ are independent of singular values of $S^{-1}$ and $S_{\mathcal{V}}^{-1}$.
\end{theorem}
\begin{proof}
We start by bounding the local error. That is, we assume that $W(t_0) = W_0^r$ and $\mathcal{V}(t_0) = \mathcal{V}_0^r$, where $\mathcal{V}_0^r$ is the momentum of the low-rank method. By definition of $\widehat U$ we have $(I-\widehat U\widehat U^{\top})\mathcal{V}(t_0) = 0$ and thus
    \begin{align*}
        \Vert (I-\widehat U\widehat U^{\top}) \mathcal{V}(t) \Vert \leq \int_{t_0}^{t}\Vert (I-\widehat U\widehat U^{\top}) (\gamma \mathcal{V}(s) + P(W(s))\nabla_W\mathcal{L}(W(s))) \Vert\,ds\,.
    \end{align*}
    Using the boundedness of normal components and a Taylor expansion around $t_0$ gives for $s\in[t_0, t_1]$
    \begin{align}\label{eq:leapprox_1}
        P(W(s))\nabla_W\mathcal{L}(W(s))) =\,& \nabla_W\mathcal{L}(W(s))) + O(\varepsilon) = \nabla_W\mathcal{L}(W(t_0)) + O(\lambda + \varepsilon)\,\nonumber\\
        =\,& P(W(t_0))\nabla_W\mathcal{L}(W(t_0)) + O(\lambda + \varepsilon)\,.
    \end{align}
    Hence, with $\mathcal{V}(s) = \mathcal{V}(t_0) + O(\lambda + \varepsilon)$,
    \begin{align*}
        \Vert (I-\widehat U\widehat U^{\top}) \mathcal{V}(t) \Vert \leq\,& \lambda\Vert (I-\widehat U\widehat U^{\top}) (\gamma \mathcal{V}(t_0) + P(W(t_0))\nabla_W\mathcal{L}(W(t_0))) \Vert + O(\lambda^2 + \lambda\varepsilon) \\
        =\,& \lambda\Vert (I-\widehat U\widehat U^{\top}) \nabla_W\mathcal{L}(W(t_0))V_0V_0^{\top} \Vert + O(\lambda^2 + \lambda\varepsilon)\,.
    \end{align*}
    By construction of $\widehat U$ we have $0 = (I-\widehat U\widehat U)\nabla_U\mathcal{L}(W(t_0)) = (I-\widehat U\widehat U)\nabla_W\mathcal{L}(W(t_0))V_0$, hence
    \begin{align*}
        \Vert (I-\widehat U\widehat U^{\top}) \mathcal{V}(t) \Vert \leq O(\lambda^2 + \lambda\varepsilon)\,.
    \end{align*}
    From this, we directly conclude
    \begin{align*}
        \Vert (I-\widehat U\widehat U^{\top}) W(t_1) \Vert = \Vert (I-\widehat U\widehat U^{\top}) (W(t_0) + \int_{t_0}^{t_1} \mathcal{V}(s)\,ds) \Vert = O(\lambda^3 + \lambda^2\varepsilon)\,.
    \end{align*}
    An analogous derivation for the co-range gives
    \begin{align*}
         \Vert W(t_1) - \widehat U\widehat U^{\top} W(t_1) \widehat V\widehat V^{\top}\Vert \leq\,& \Vert (I - \widehat U\widehat U^{\top}) W(t_1) \Vert + \Vert W(t_1) (I - \widehat V\widehat V^{\top}) \Vert \\  
         =\,& O(\lambda^3 + \lambda^2\varepsilon)\,.
    \end{align*}
    Next, we need to bound 
    \begin{align}
        \Vert \widehat U\widehat U^{\top} W(t_1) \widehat V\widehat V^{\top} - \widehat U\widehat S^1\widehat V^{\top} \Vert \leq\,& \Vert \widehat U^{\top} W(t_1) \widehat V - \widehat S^1 \Vert\,.
    \end{align}
    We note that from {Eq.} \eqref{eq:leapprox_1} we have with $W_0 := W(t_0)$ and $\mathcal{V}_0 := \mathcal{V}(t_0)$
    \begin{align*}
        \widehat U^{\top} W(t_1) \widehat V =\,& \widehat U^{\top} (W_0+\lambda(1-\gamma) \mathcal{V}_0 - \lambda^2 P(W_0)\nabla_W\mathcal{L}(W_0)) \widehat V + O(\lambda^2 + \lambda \varepsilon) \\
        =\,&  \bar S-\lambda\gamma \bar S_{\mathcal{V}} - \lambda \widehat U^{\top}\nabla_W\mathcal{L}(W_0) \widehat V + O(\lambda^2 + \lambda \varepsilon)\,,
    \end{align*}
    where $\bar S = \widehat U^{\top} W_0\widehat V$ and $\bar S_{\mathcal{V}} = \widehat U^{\top} \mathcal{V}_0\widehat V$. By definition of the $S$-update of Algorithm~\ref{alg:adam_dlrt} we have
    \begin{align*}
        \widehat S^1 = \bar S + \lambda (1-\gamma) \bar S_{\mathcal{V}} - \lambda^2 \nabla_{\bar S}\mathcal{L}(\widehat U \bar S \widehat V^{\top})\,.
    \end{align*}
    Thus, since $\nabla_{\bar S}\mathcal{L}(\widehat U \bar S \widehat V^{\top}) = U^{\top}\nabla_W\mathcal{L}(W_0) \widehat V$ we have $\Vert \widehat U^{\top} W(t_1) \widehat V - \widehat S^1 \Vert = O(\lambda^2 + \lambda \varepsilon)$ and therefore the local error is bounded by
    \begin{align*}    
        \Vert{W(t_1)-W^r_1}\Vert \leq\,& \Vert W(t_1) - \widehat U\widehat U^{\top} W(t_1) \widehat V\widehat V^{\top}\Vert+\Vert \widehat U\widehat U^{\top} W(t_1) \widehat V\widehat V^{\top} - \widehat U\widehat S^1\widehat V^{\top} \Vert\\
        =\,& O(\lambda^2 + \lambda \varepsilon)\,.
    \end{align*}
    From the truncation tolerance $\vartheta$, the bound on the truncation of $\mathcal{V}$, and the stability of the exact flow, we can obtain the desired error bound for the global error using Lady Windermere's fan.
\end{proof}
We remark, that we can always ensure that condition $\Vert \widehat U_t\widehat U_t^{\top}\mathcal{V}_t\widehat V_t\widehat V_t^{\top} - U_tU_t^{\top}\mathcal{V}_tV_tV_t^{\top}\Vert\leq\widehat \vartheta$, is fulfilled for a user determined $\widehat \vartheta$, e.g. $\widehat \vartheta=\vartheta$,  by increasing the new rank $r_1$ in the truncation step of \Cref{alg_helper} if necessary. {However, since $\mathcal{V}\rightarrow 0$ when the method reaches a steady state, the effect of this error term is expected to be limited. We remark that the main motivation to present Theorem~\ref{th:robust_error_bound} is to rigorously demonstrate the robust treatment of stiff terms in the gradient flow~\eqref{eq:momentum_efficient}. The main component in our construction of the algorithm, which removes these stiff terms in our error bound, is the construction of the augmented basis matrices $\widehat U$ and $\widehat V$.  }

\begin{theorem}\label{th:failedConv}
    The conventional low-rank gradient flow equations \eqref{eq:vanilla_grad_flow} can fail to converge to a point fulfilling $P(W)\nabla_W \mathcal{L}(W)=0$.
\end{theorem}
\begin{proof}
    The potential lack of convergence can be proven with a counter-example. Consider a state where the momentum term is zero, i.e., $\mathcal{V}=0$, by choosing the momentum factors as $U_{\mathcal{V}} = -U$, $S_{\mathcal{V}} = 0$, and $V_{\mathcal{V}} = V$. Now, consider any weight matrix $W$ that is not a low-rank optimum, meaning $P(W)\nabla_W \mathcal{L}(W) \neq 0$, but for which the naive update term in Eq. \eqref{eq:naiveUppdateFlowMomentum} happens to be zero: $\hat{P}(W,\mathcal{V})\nabla_W \mathcal{L}(W) = 0$.
    In this scenario, the naive gradient flow equations would identify $(W, \mathcal{V})$ as a stationary point, since $\dot{W}=0$ and $\dot{\mathcal{V}}=0$. However, this point is not a valid low-rank optimum because the true projected gradient $P(W)\nabla_W \mathcal{L}(W)$ is non-zero. In contrast, $(W, \mathcal{V})$ is not a stationary point of (4) and the gradient flow (4) will provably drive the system to a state where $P(W)\nabla_W \mathcal{L}(W)=0$.
\end{proof}
We remark that a large class of matrices fulfills $P(W)\nabla_W \mathcal{L}(W) \neq 0$ and $\hat{P}(W,\mathcal{V})\nabla_W \mathcal{L}(W) = 0$. For instance, any $W=USV^\top$ where the gradient is non-zero in the range of $V$ but zero in the range of $U$ (i.e., $U^\top\nabla_W \mathcal{L}(W) = 0$ but $\nabla_W \mathcal{L}(W)V \neq 0$) would satisfy this. This can be easily verified: Since $\mathcal{V}=0$, we have
\begin{align*}
    \hat{P}(W,\mathcal{V})\nabla_W \mathcal{L}(W) = UU^\top\nabla_W \mathcal{L}(W)VV^\top = 0\,,
\end{align*}
but
\begin{align*}
P(W)\nabla_W \mathcal{L}(W) = \nabla_W \mathcal{L}(W)VV^{\top}\neq 0\,.
\end{align*}

\section{Details to the numerical experiments of this work}\label{app_experiments}
{It is important to note that the accuracy-vs-compression trade-off varies by application. While low-rank methods excel in finetuning and transfer learning tasks (sometimes even improving upon the baseline), pre-training a network from scratch on a complex dataset often involves balancing memory savings against a potential drop in accuracy.}

\subsection{ImageNet-1k, UCM and Cifar Benchmarks}\label{sec_details_UCM}

\subsubsection{Network architecture and training details}\label{sec_network_training_details_vision}
In this paper, we use the pytorch implementation for neural network training. We take pretrained weights from the imagenet1k dataset as initialization, except for the long-term training study using ViT-small, which is randomly initialized. The data-loaded randomly samples a batch for each batch-update which is the only source of randomness in our training setup. Below is an overview of the used network architectures
\begin{itemize}[leftmargin=*, nosep]
    \item VGG16 is a deep convolutional neural network architecture that consists of 16 layers, including 13 convolutional layers and 3 fully connected layers.  
    \item VGG11 is a convolutional neural network architecture similar to VGG16 but with fewer layers, consisting of 11 layers: 8 convolutional layers and 3 fully connected layers. It follows the same design principle as VGG16, using small 3×3 convolution filters and 2×2 max-pooling layers. 
    \item ViT-B.16 is a Vision Transformer with $16\times16$ patch size, a deep learning architecture that leverages transformer models for image classification tasks.
    \item ViT-small is a compact vision transformer with patch size $8 \times 8$, and an embedding dimension of 512. The model comprises six attention layers, each equipped with two heads, followed by a ResNet block and a dropout layer. 
    \item {ViT-L.32 is  a Vision Transformer with 32x32 patch size, a deep learning architecture that leverages transformer models for image classification tasks. We use the Imagenet21k weights from the huggingface endpoint google/vit-large-patch32-224-in21k as weight initialization.}
\end{itemize}

The full training setup is described in \Cref{tab_UCM_trainings}.  We train DLRT with the same hyperparameters as the full-rank baseline models. It is known \cite{schotthöfer2024GeoLoRAgeometricintegrationparameter} that DLRT methods are robust w.r.t. common hyperparameters as learning rate, and batch-size, and initial rank. 
The truncation tolerance $\tau$ is chosen per an initial parameter study. These values are are similar to default values reported in recent literature \cite{schotthöfer2024federateddynamicallowranktraining, schotthöfer2025dynamicallowrankcompressionneural, singh2021analytic}.
In general, there is a trade-off between target compression ratio and accuracy, as illustrated e.g. in \cite{schotthofer2022low} for matrix-valued and \cite{singh2021analytic} for tensor-valued (CNN) layers.

\begin{table}[t]
\centering
\caption{Training hyperparameters for the UCM, Cifar10, Cifar100 and ImageNet1k Benchmark. The first set hyperparameters apply to both DLRT and baseline training, and we train DLRT with the same hyperparameters as the full-rank baseline models. The second set of hyper-parameters is specific to DLRT. The DLRT hyperparameters are selected by an initial parameter sweep. We choose the DLRT truncation tolerance relative to the Frobenius norm of $\augS$, i.e. $\vartheta=\tau \|\augS\|_F$, as suggested in \cite{schotthofer2022low}.}
\label{tab_UCM_trainings}
\begin{tabular}{lccccc}
\toprule
\textbf{Hyperparameter} & \textbf{VGG16} & \textbf{VGG11} & \textbf{ViT-B.16} & \textbf{ViT-small} & \textbf{ViT-L.32} \\
\midrule
Batch Size (UCM)                  & 16  & 16  & 16   & n/a & n.a. \\
Batch Size (Cifar10)              & 128 & 128 & 128  & 256 & n.a. \\
Batch Size (Cifar100)             & 30  & 30  & 20   & n.a & n.a. \\
Batch Size (ImageNet)             & n.a & n.a & n.a  & n.a & 256 \\
Learning Rate                     & 0.001 & 0.001 & 0.001 & 0.0001 & 0.001 \\
Number of Epochs (UCM, Cifar10)  & 20 & 20 & 5 & 450 &  n.a  \\
Number of Epochs (Cifar100)      & 30 & 30 & 20 & n.a & n.a \\
Number of Epochs (ImageNet1k)      &  n.a  &  n.a  &  n.a  & n.a & 10 \\
L2 Regularization                  & 0 & 0 & 0.001 & 0.01 & 0.0001 \\
Optimizer                          & AdamW & AdamW & AdamW & Adam & AdamW \\
\midrule
DLRT rel. truncation tolerance $\tau$ & 0.1 & 0.05 & 0.08 & 0.05 & 0.013 \\
Coefficient Steps $s_*$             & 10  & 10  & 10   & 75 & 75 \\
Initial Rank                        & 150 & 150 & 150  & 200 & 200 \\
\midrule
Parameters                          & 138M & 132M & 86M & 50M & 304M \\
\bottomrule
\end{tabular}
\end{table}

\subsubsection{UCM Data}
The UC Merced (UCM) Land Use Dataset~\cite{10.1145/1869790.1869829} is a standard benchmark in remote sensing and computer vision. It consists of 2,100 high-resolution aerial RGB images, each of size $256 \times 256$ pixels, organized into 21 land use classes with 100 images per class.

We normalize the training and validation data using channel-wise means $[0.485, 0.456, 0.406]$ and standard deviations $[0.229, 0.224, 0.225]$. Convolutional neural networks (CNNs) are applied directly to the original $256 \times 256$ image resolution. For the Vision Transformer (ViT), the input images are resized to $224 \times 224$ pixels within the data pipeline.

\subsubsection{CIFAR-10 Data}\label{sec_details_Cifar10}
The CIFAR-10 dataset comprises 60,000 RGB images of size $32 \times 32$ pixels, uniformly distributed across 10 object classes.

We apply standard data augmentation techniques to the training set, including random horizontal flipping followed by normalization with mean $[0.4914, 0.4822, 0.4465]$ and standard deviation $[0.2470, 0.2435, 0.2616]$. The test set is only normalized. The same augmentation strategy is applied to CIFAR-100, using mean $[0.5071, 0.4867, 0.4408]$ and standard deviation $[0.2673, 0.2564, 0.2762]$.

CNNs are trained on the original $32 \times 32$ resolution, while ViT models receive images resized to $224 \times 224$ through the data pipeline.

\subsubsection{ImageNet-1k Data}
{The ImageNet dataset consists of 1000 classes and over 1.2 million RGB training images, with a standard resolution of $224 \times 224$ pixels. We follow the standard data augmentation pipeline for ImageNet, which includes a random resized crop to $224 \times 224$, and normalization using mean $[0.5, 0.5, 0.5]$ and standard deviation $[0.5, 0.5, 0.5]$. The test set is only resized and center-cropped to $224 \times 224$, followed by normalization.}

\subsubsection{Additional Results - Transfer Learning} 
\begin{wraptable}{r}{0.4\textwidth}
\centering
\caption{Results on ImageNet-1k with ViT-L.32 (304M parameters). Compression rate (c.r.) is reported in percent.}
\label{tab:imagenet}
\resizebox{0.4\textwidth}{!}
{%
\begin{tabular}{lccc}
\toprule
\textbf{Method} & \textbf{c.r. [\%]} & \textbf{Top-1 Acc. [\%]} & \textbf{Top-5 Acc. [\%]} \\
\midrule
Baseline       & 0     & \textbf{74.37} & {92.20} \\
\midrule
\Cref{alg:adam_dlrt}   & 61.45 & 72.27 & \textbf{90.19} \\
LoRA Pretrain  & 60.00 & 63.20 & 84.81 \\
\bottomrule
\end{tabular}
}
\end{wraptable}

\begin{figure}[t]
\begin{subfigure}{0.49\linewidth}
        \centering
    \includegraphics[width=\linewidth]{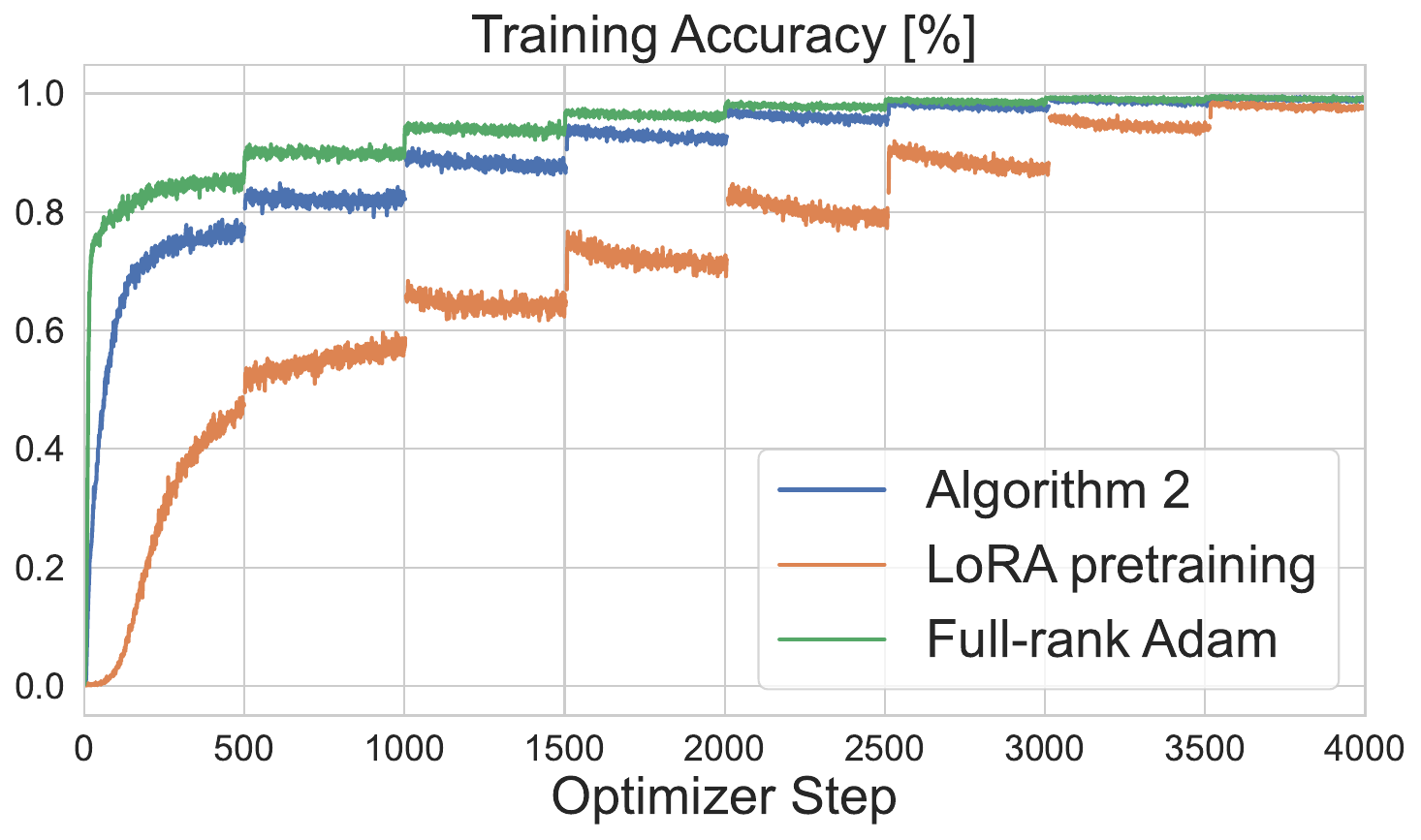}
\end{subfigure}
\begin{subfigure}{0.49\linewidth}
    \centering
    \includegraphics[width=\linewidth]{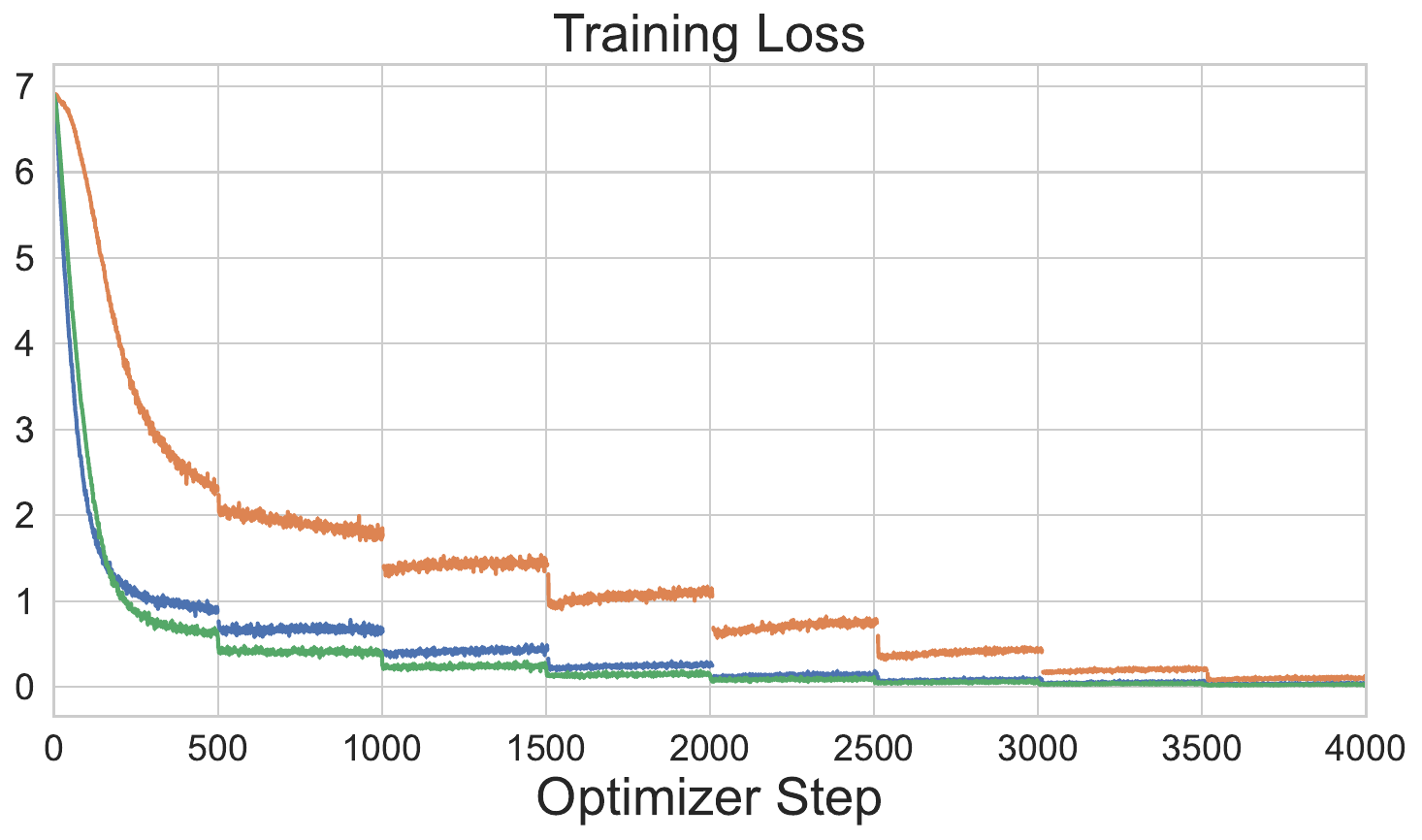}
\end{subfigure}
\caption{{ViT-L.32 on ImageNet1k, pretrained from scratch in low-rank and full-rank baseline format for 4000 iterations. Training loss and accuracy of \Cref{alg:adam_dlrt} is close to the full-rank baseline, whereas LoRA pretraining struggles to converge within the training time budget.}
}\label{fig_conv_behavior_imgnet1k}
\vspace{-0.27cm}
\end{figure}

\paragraph{ViT-L.32 on ImageNet1k} We repeat the experimental setup on ViT-L.32 on the ImageNet-1k dataset, where ViT-L.32 is initialized with a Huggingface Imagenet21K checkpoint. We compare the baseline model, LoRA-based simultaneous descent pretraining, and  \Cref{alg:adam_dlrt} in Table~\ref{tab:imagenet}. We observe that  \Cref{alg:adam_dlrt} is able to recover the baseline accuracy up to a small margin whereas LoRA-based training exhibits decreased Top-1 and Top-5 accuracy.
Finally, we remark that the slightly lower compression rate is expected since the hidden dimension of ViT-L.32 ($1024$) is close to the number of ImageNet classes ($1000$), thus there is less redundancy in the model compared to other reported benchmarks.

\subsubsection{Additional Results - Transfer Learning: Low-Rank Heavyball Method }
\begin{wraptable}{r}{0.4\textwidth}
\centering
\caption{Results on Cifar10 with VGG16 using low-rank Heavyball SGD. Compression rate (c.r.) is reported in percent.}
\label{tab:heavyball}
\resizebox{0.4\textwidth}{!}
{%
\begin{tabular}{lccc}
\toprule
\textbf{Method} & \textbf{c.r. [\%]} & \textbf{Acc. [\%]}\\
\midrule
Baseline       & 0     & 78.98 \\
\midrule
\Cref{alg:heavy_ball_dlrt}   & 94.35 & 	\textbf{79.01}  \\
LoRA transfer learning  & 93.72 & 75.12  \\
\bottomrule
\end{tabular}
}
\end{wraptable}
\paragraph{VGG16 on Cifar10}
We consider VGG16 on Cifar10 with Heavyball SGD using the same hyperparameters as described in \Cref{sec_network_training_details_vision}. We choose $\gamma$=0.9 and train a (full-rank) baseline, LoRA pretrain and \Cref{alg:heavy_ball_dlrt}. The compression rate of LoRA pretrain is fixed to match the final compression rate of \Cref{alg:heavy_ball_dlrt}. In \Cref{tab:heavyball}, we observe similar performance of \Cref{alg:heavy_ball_dlrt} to the Baseline as we saw for \Cref{alg:adam_dlrt} to the Adam Baseline, whereas LoRA pretrain exhibits a slight drop in accuracy.

\subsubsection{Additional Results - Low-Rank Pretraining}
\paragraph{ViT-small on Cifar10}

We consider a compact Vision Transformer architecture for the CIFAR-10 dataset, see \Cref{sec_details_UCM} for details on training and architecture. We compare baseline full-rank training with LoRA pretraining, \Cref{alg:adam_dlrt}, and the naive implementation of Adam with DLRT \citep{schotthofer2022low}. Instead of low-rank finetuning, we pretrain the network from scratch and initialize the weights with a normal distribution. The low-rank methods factorize the fully-connected layers, while keeping the attention layers, which are typically high-rank, in baseline format. The LoRA rank is chosen to match the final compression rate of the rank adaptive naive DLRT and \Cref{alg:adam_dlrt} method, which achieves a compression rate of $67\%$. Remark that the compression rate is lower, since the attention matrices remain full-rank. 

\begin{figure}
\begin{subfigure}{0.33\linewidth}
        \centering
    \includegraphics[width=\linewidth]{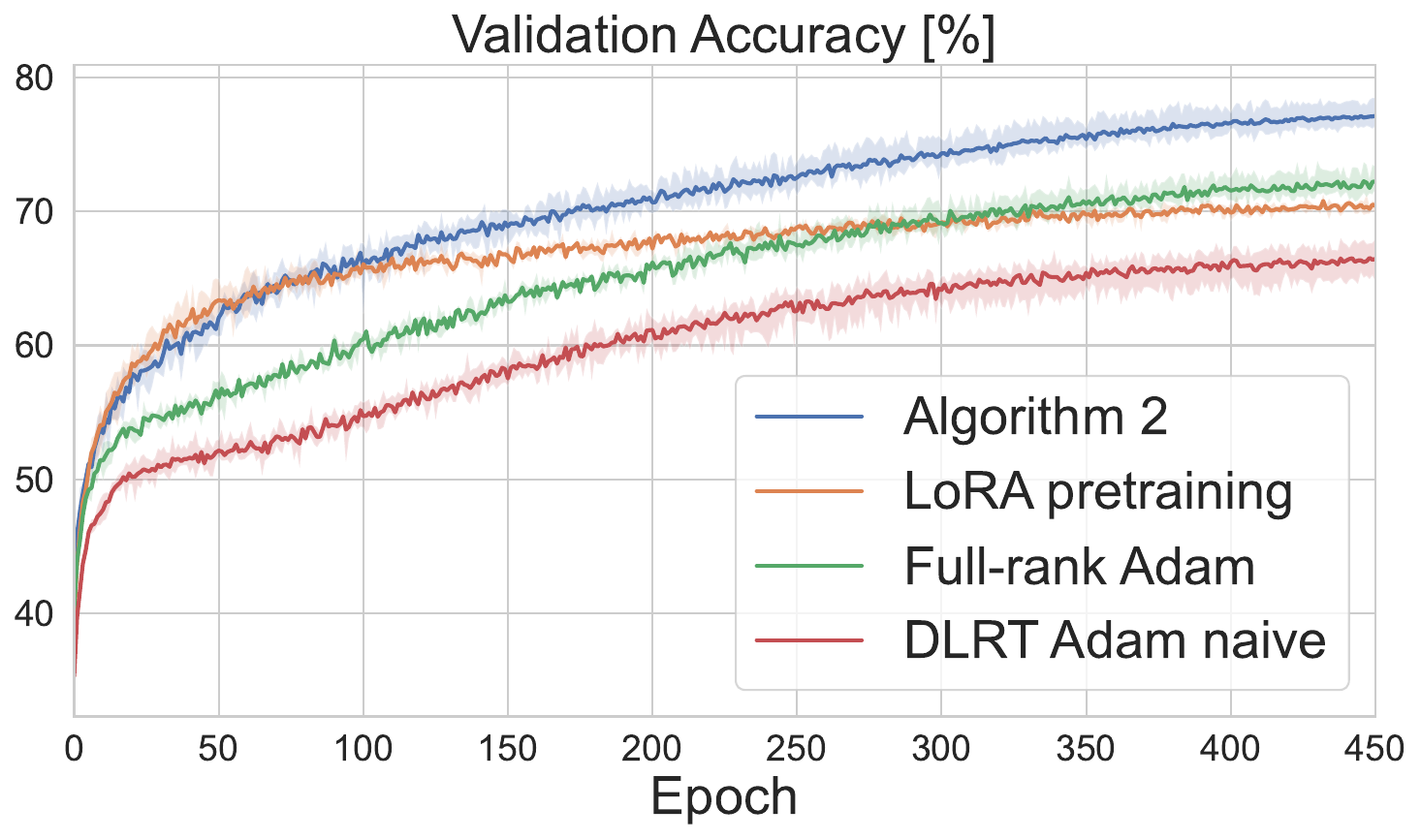}

\end{subfigure}
\begin{subfigure}{0.33\linewidth}
    \centering
    \includegraphics[width=\linewidth]{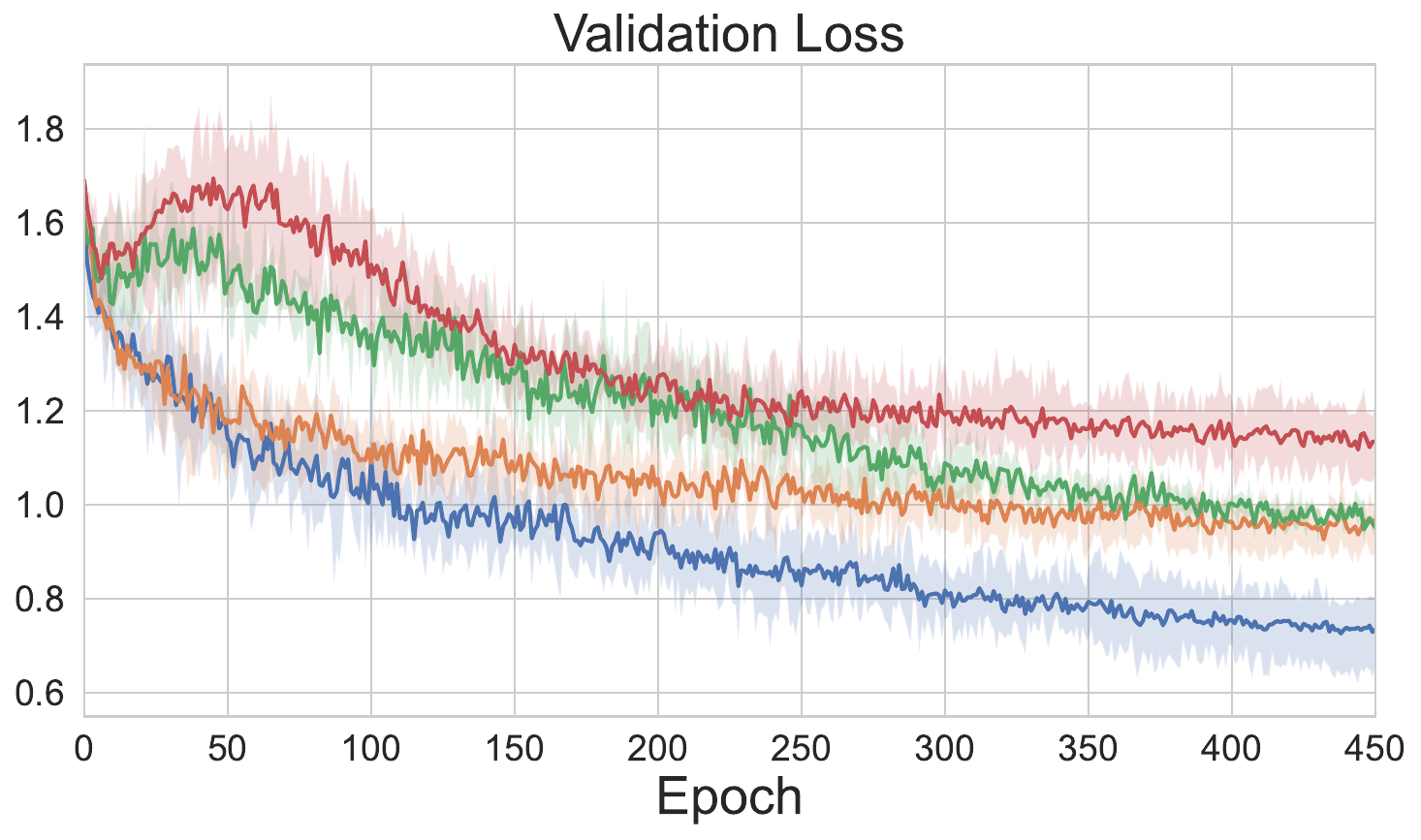}

\end{subfigure}
\begin{subfigure}{0.33\linewidth}
    \centering
    \includegraphics[width=\linewidth]{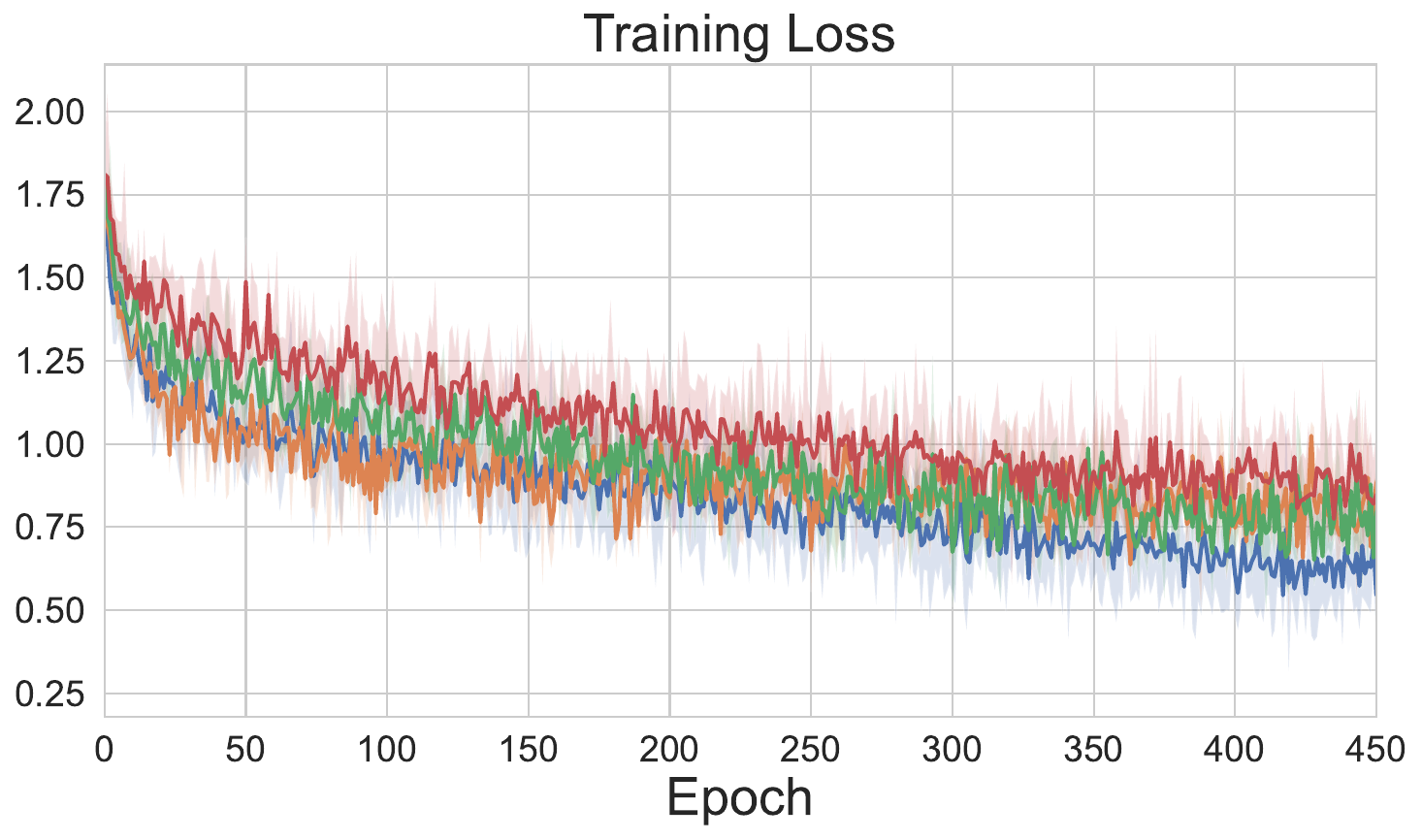}
\end{subfigure}
\caption{ViT-small on Cifar10, pretrained from scratch in low-rank and full-rank baseline format for 450 epochs. Median trajectory over 5 runs. \Cref{alg:adam_dlrt} and LoRA pretraining initially converge faster than the full-rank baseline. {After the initial warm-up phase, \Cref{alg:adam_dlrt} exhibits a steeper convergence slope than LoRA. Moreover,} \Cref{alg:adam_dlrt} achieves lower loss and higher validation accuracy than LoRA, even surpassing the baseline. A naive DLRT implementation with Adam leads to slower convergence and over 10\% drop in validation accuracy.
}\label{fig_conv_behavior}
\vspace{-0.27cm}
\end{figure}

The goal of this test is to compare the long-term convergence behavior of all four methods, presented in \Cref{fig_conv_behavior}. We observe that \Cref{alg:adam_dlrt} and standard LoRA pretraining first converge faster than the baseline training. The non-orthogonal bases $A,B$ of LoRA and the corresponding non-orthogonal projection onto the low-rank manifold cause LoRA to plateau, whereas \Cref{alg:adam_dlrt} achieves lower loss values and higher validation accuracies, even surpassing the full-rank baseline in this test case. Naive implementation of DLRT with the Adam optimizer causes a more than 10\% reduction in validation accuracy and slower convergence.

\subsection{GLUE Benchmark}\label{app_glue}
\subsubsection{Dataset description}
We present the benchmark overview in \Cref{tab_glue_overview}.
\begin{table}[h]
\caption{Summary of GLUE benchmark tasks}
\label{tab_glue_overview}
\centering
\begin{tabular}{@{}l l r r r r l@{}}
\toprule
\textbf{Corpus} & \textbf{Task} & \textbf{\#Train} & \textbf{\#Dev} & \textbf{\#Test} & \textbf{\#Label} & \textbf{Metrics} \\ 
\midrule
\multicolumn{7}{c}{\textbf{Single-Sentence Classification (GLUE)}} \\ 
CoLA & Acceptability & 8.5k & 1k & 1k & 2 & Matthews corr \\
SST2  & Sentiment & 67k & 872 & 1.8k & 2 & Accuracy \\ 
\midrule
\multicolumn{7}{c}{\textbf{Pairwise Text Classification (GLUE)}} \\ 
MNLI & NLI & 393k & 20k & 20k & 3 & Accuracy \\
RTE & NLI & 2.5k & 276 & 3k & 2 & Accuracy \\
QQP & Paraphrase & 364k & 40k & 391k & 2 & F1 \\
MRPC & Paraphrase & 3.7k & 408 & 1.7k & 2 & Accuracy \\
QNLI & QA/NLI & 108k & 5.7k & 5.7k & 2 & Accuracy \\
\midrule
\multicolumn{7}{l}{\textbf{Text Similarity (GLUE)}} \\ 
STS-B & Similarity & 7k & 1.5k & 1.4k & 1 & Pearson/Spearman cor \\
\bottomrule
\end{tabular}
\end{table}
We evaluate \ALGNAME{} against several recent finetuning methods on the General Language Understanding Evaluation (GLUE) benchmark~\citep{wang2019gluemultitaskbenchmarkanalysis}.
GLUE is a standard benchmark comprising a diverse set of natural language understanding tasks that assess a model’s ability to comprehend and process human language. It provides a broad evaluation by including tasks covering various linguistic aspects such as entailment, sentiment, and semantic similarity. The benchmark comprises the following nine tasks:
\begin{itemize}[leftmargin=*,noitemsep,topsep=0em]
    \item \textbf{CoLA (Corpus of Linguistic Acceptability)}: Determines if a sentence is grammatically acceptable.
    \item \textbf{SST-2 (Stanford Sentiment Treebank)}: A binary sentiment classification task distinguishing between positive and negative sentiment.
    \item \textbf{MRPC (Microsoft Research Paraphrase Corpus)}: Identifies whether two given sentences are paraphrases.
    \item \textbf{STS-B (Semantic Textual Similarity Benchmark)}: Measures the semantic similarity of two sentences on a continuous scale from 1 to 5.
    \item \textbf{QQP (Quora Question Pairs)}: Assesses whether two questions are semantically equivalent.
    \item \textbf{QNLI (Question Natural Language Inference)}: Determines if a context sentence correctly answers a question.
    \item \textbf{RTE (Recognizing Textual Entailment)}: A binary entailment classification task.
    \item \textbf{Specific Focus:} MRPC (Microsoft Research Paraphrase Corpus)
\end{itemize}

The F1 score, used for evaluation, is computed from the precision $P$ and recall $R$ as follows.
The precision $P$ is defined as
\begin{align}
    P := \frac{P_T}{P_T + P_F},
\end{align}
where $P_T$ denotes the number of true positives and $P_F$ the number of false positives. The recall $R$ is given by
\begin{align}
    R := \frac{P_T}{P_T + N_F},
\end{align}
where $N_F$ represents the number of false negatives. The F1 score is then the harmonic mean of $P$ and $R$:
\begin{align}
    F1 := \frac{2PR}{P + R}.
\end{align}

\subsubsection{Reference implementations}
\textbf{Full Finetuning (FT)}: The standard approach in transfer learning, where the model is initialized with pre-trained weights and all parameters are updated via gradient descent.

\textbf{Bitfit~\citep{zaken2022bitfitsimpleparameterefficientfinetuning}}: Finetuning where only the bias terms are updated while all other parameters remain fixed.

\textbf{Adapter Tuning~\citep{pmlr-v97-houlsby19a,pfeiffer2021adapterfusionnondestructivetaskcomposition}}: Involves inserting two-layer adapter modules within transformer blocks. In~\citep{pmlr-v97-houlsby19a}, adapters are placed between the self-attention and feed-forward modules with a residual connection (denoted HAdapter). In~\citep{pfeiffer2021adapterfusionnondestructivetaskcomposition}, adapters are inserted after the feed-forward and layer normalization modules (denoted PAdapter), following the notation of~\citep{zhang2023adalora}.

\textbf{LoRA~\citep{hu2021lora}}: Applies low-rank additive updates to selected weight matrices, modeled as
\begin{align}
    \fz = \sigma\left(W_{\mathrm{pt}}\fx + \frac{\alpha}{r}AB^\top\fx\right),
\end{align}
where $A, B \in \mathbb{R}^{n \times r}$. We apply LoRA to the attention matrices $W_q$, $W_k$, $W_v$, and the feed-forward matrices $W_{f_1}$ and $W_{f_2}$. Learning rates and optimizers follow the setup in~\citep{zhang2023adalora}, Appendix D–F.

{{Results for FT, Bitfit, Adapter tuning, and LoRA in \Cref{tab_results} are reproduced from~\citep{zhang2023adalora}. The performance of DoRA, LoRA, LoRA+, and AdaLoRA is computed using the HuggingFace implementations of these adapters.}}

\textbf{DoRA~\citep{Mao2024DoRAEP}}: A low-rank adapter similar in structure to LoRA, but with normalized $AB$ matrices and an additional magnitude parameter. Unlike LoRA, DoRA initializes the adapter with the pre-trained weights $W_0$, rather than zero.

\textbf{LoRA+~\citep{hayou2024lora}}: Differs from LoRA in the assignment of learning rates: separate learning rates are used for $A$ and $B$, with a fixed ratio $\lambda_B/\lambda_A = 1.1$.

\textbf{AdaLoRA~\citep{zhang2023adalora}}: Introduces adaptive low-rank updates to selected weight matrices:
\begin{align}
    \fz = \sigma\left(W_{\mathrm{pt}}\fx + \frac{\alpha}{r}USV^\top\fx\right),
\end{align}
with frozen base weights $W_{\mathrm{pt}} \in \mathbb{R}^{n \times n}$, rank-$r$ adapters $U, V \in \mathbb{R}^{n \times r}$, and scaling matrix $S \in \mathbb{R}^{r \times r}$. The rank is determined using either SVD-based truncation or sensitivity analysis of the singular vectors. AdaLoRA is applied to $W_q$, $W_k$, $W_v$, $W_{f_1}$, and $W_{f_2}$ with an orthogonality regularization coefficient $\gamma = 0.1$.

{{When comparing to AdaLoRA, we align the total parameter budget with LoRA by setting the final budget $b^{(T)}$ to $576$, and initialize with $b^{(0)} = 1.5 \times b^{(T)}$.}}

{{We also compare AdaLoRA using budget schedules obtained via \Cref{alg:adam_dlrt}, ensuring that $b^{(T)}$ approximately matches the parameter count of the final models trained using \Cref{alg:adam_dlrt}.}}

\textbf{GeoLoRA~\citep{schotthöfer2024GeoLoRAgeometricintegrationparameter}}: GeoLoRA integrates the projected gradient flow \Cref{eq:grad_flow_dlr} in a parallelizable single-step scheme, including a rank adapative augmentation-truncation scheme as the proposed method. However, the method is only applicable for stochastic gradient descent, and not yet extended to momentum-based approaches. We use the hyperparameter choices reported in \citep{schotthöfer2024GeoLoRAgeometricintegrationparameter}.

We use the implementation of \citep[Appendix C]{zhang2023adalora} to compute the results for the presented reference methods. 
We set the exponential moving average parameters $\beta_1$ and $\beta_2$ of AdamW
as their pytorch default value. We select the learning rates as denoted in \Cref{tab_glue_hyperparam}, selected by an initial hyperparameter sweep.

We implement \ALGNAME{} as similar as possible to the reference models to achieve a fair comparison. That is, we add an adapter of the form $\fz = \sigma(W_{\mathrm{pt}}\fx + USV^\top\fx)$  to the key $W_k$, query $W_q$ and value $W_v$ matrices of all attention blocks, and to both feed-forward layers $W_{f_1}$ and $W_{f_2}$. For each adapter, we employ \Cref{alg:adam_dlrt} to update the layer weights and ranks.

\begin{table}[h!]
\centering
\caption{Hyper-parameter setup for the GLUE benchmark, determined by an initial hyperparameter sweep.}
\label{tab_glue_hyperparam}
\resizebox{\textwidth}{!}{
\begin{tabular}{l|ccccccc}
\toprule
\textbf{Dataset} & \textbf{Learning Rate} & \textbf{Batch Size} & \textbf{\# Epochs} & \textbf{$\tau$} & \textbf{init. rank}  & \textbf{Adapter dropout} & \textbf{weight decay}\\
\midrule
RTE    & $1.2 \times 10^{-3}$   & 32 & 20 & 0.075 &10 & 0.01 & 0.01\\
QNLI   & $5 \times 10^{-4}$   & 64 & 5 & 0.05 &10 & 0.2 & 0.01\\
MRPC   & $1 \times 10^{-4}$   & 64 & 5 & 0.05 &10 & 0.15 & 0.05\\
QQP    & $1 \times 10^{-4}$   & 64 & 5 & 0.05 &10 & 0.15 & 0.05\\
SST-2  & $1 \times 10^{-4}$   & 64 & 10 & 0.05 &10 & 0.05 & 0.01\\
CoLA   & $5 \times 10^{-4}$   & 32 & 25 & 0.05 &10 & 0.1 & 0.01\\
STS-B & $1 \times 10^{-3}$   & 128 & 30 & 0.05 &10 & 0.05 & 0.1\\
\bottomrule
\end{tabular}
}
\end{table}

\subsection{Llama2 7b-chat-hf on BoolQ and PIQA}
{
\textbf{BoolQ} is a reading comprehension dataset consisting of naturally occurring yes/no questions paired with passages from Wikipedia. Questions are drawn from real Google search queries, and each is annotated with an answer by human raters, making it a benchmark for natural, open-domain question answering.
}

{
\textbf{PIQA} (Paragraph-level In-context QA) is a dataset designed for evaluating in-context learning in long-form reading comprehension. It provides paragraph-length passages with associated questions and answers, emphasizing models’ ability to extract relevant information from extended contexts rather than isolated sentences.
}

\begin{table}[h!]
\centering
\caption{Hyper-parameter setup for \Cref{alg:adam_dlrt} for the reasoning benchmark Table \ref{tab:Llama}, determined by an initial hyperparameter sweep.}
\label{tab_llama_hyperparam}
\resizebox{\textwidth}{!}{
\begin{tabular}{l|ccccccc}
\toprule
\textbf{Dataset} & \textbf{Learning Rate} & \textbf{Batch Size} & \textbf{\# Epochs} & \textbf{$\tau$} & \textbf{init. rank}  & \textbf{Adapter dropout} & \textbf{weight decay}\\
\midrule
BoolQ &  $1.76 \times e^{-4}$ & 12 & 3 & 0.0696   & 6 &0 &0.1\\
PIQA  &  $1.36 \times e^{-4}$ & 12 & 3 & 0.0838   & 6 &0 &0.1\\
\bottomrule
\end{tabular}
}
\end{table}

\begin{table}[h!]
\centering
\caption{Hyper-parameter setup for LoRA for the reasoning benchmark Table \ref{tab:Llama}, determined by an initial hyperparameter sweep.}
\label{tab_LoRA_llama_hyperparam}
\resizebox{\textwidth}{!}{
\begin{tabular}{l|ccccccc}
\toprule
\textbf{Dataset} & \textbf{Learning Rate} & \textbf{Batch Size} & \textbf{\# Epochs} & \textbf{$\tau$} & \textbf{init. rank}  & \textbf{Adapter dropout} & \textbf{weight decay}\\
\midrule
BoolQ &  $4.47 \times e^{-4}$ / $1.76 \times e^{-4}$ & 12 & 3 & None   & 6/10 &0 &0.1\\
PIQA  &  $2.04 \times e^{-4}$ / $1.36 \times e^{-4}$ & 12 & 3 & None   & 6/10 &0 &0.1\\
\bottomrule
\end{tabular}
}
\end{table}

\subsection{GPT2 on OpenWebText}
{
\textbf{OpenWebText} is an open-source dataset constructed as a replication of OpenAI’s WebText. It was created by scraping URLs shared on Reddit. The dataset contains web pages spanning diverse topics, filtered to remove duplicates and non-English text, and is commonly used as a large-scale corpus for training and evaluating language models.
}
\begin{table}[h!]
\centering
\caption{Hyperparameter configuration for pretraining GPT-2 (124M) on OpenWebText (see Table~\ref{tab:GPT2}).}
\label{tab_GPT2_hyperparam}
\resizebox{\textwidth}{!}{
\begin{tabular}{l|ccccccc}
\toprule
\textbf{Dataset} & \textbf{Learning Rate} & \textbf{Batch Size} & \textbf{\# iteration} & \textbf{$\tau$} & \textbf{init. rank}  & \textbf{Adapter dropout} & \textbf{weight decay}\\
\midrule
OpenWebText & $[ 6e^{-4}, 6e^{-5}]$ & 64 & 15 000 & 0.05  & 135 & 0 & 0.1\\
\bottomrule
\end{tabular}
}
\end{table}

\begin{figure}[t]
\begin{subfigure}{0.48\linewidth}
        \centering
    \includegraphics[width=\linewidth]{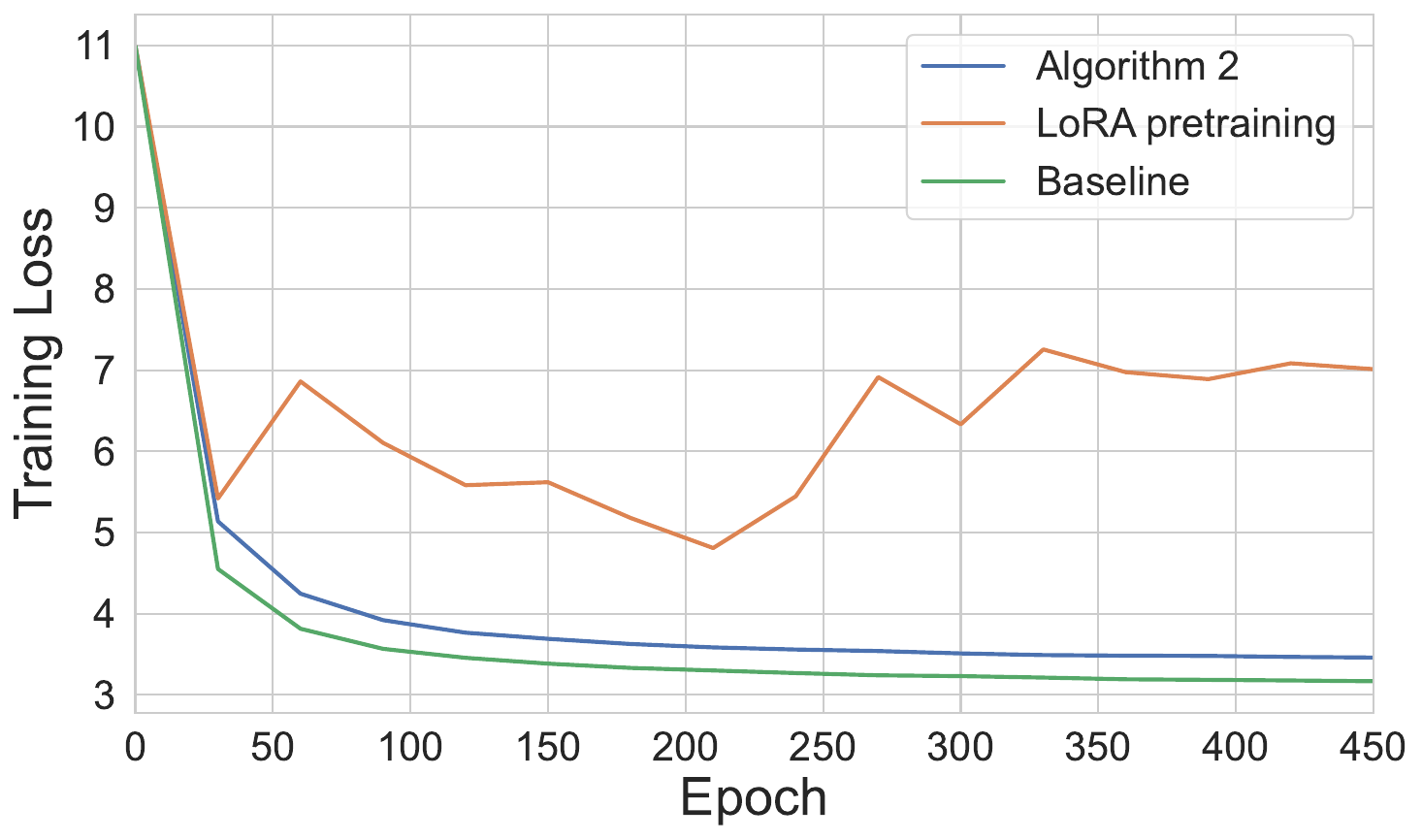}
\end{subfigure}
\begin{subfigure}{0.48\linewidth}
    \centering
    \includegraphics[width=\linewidth]{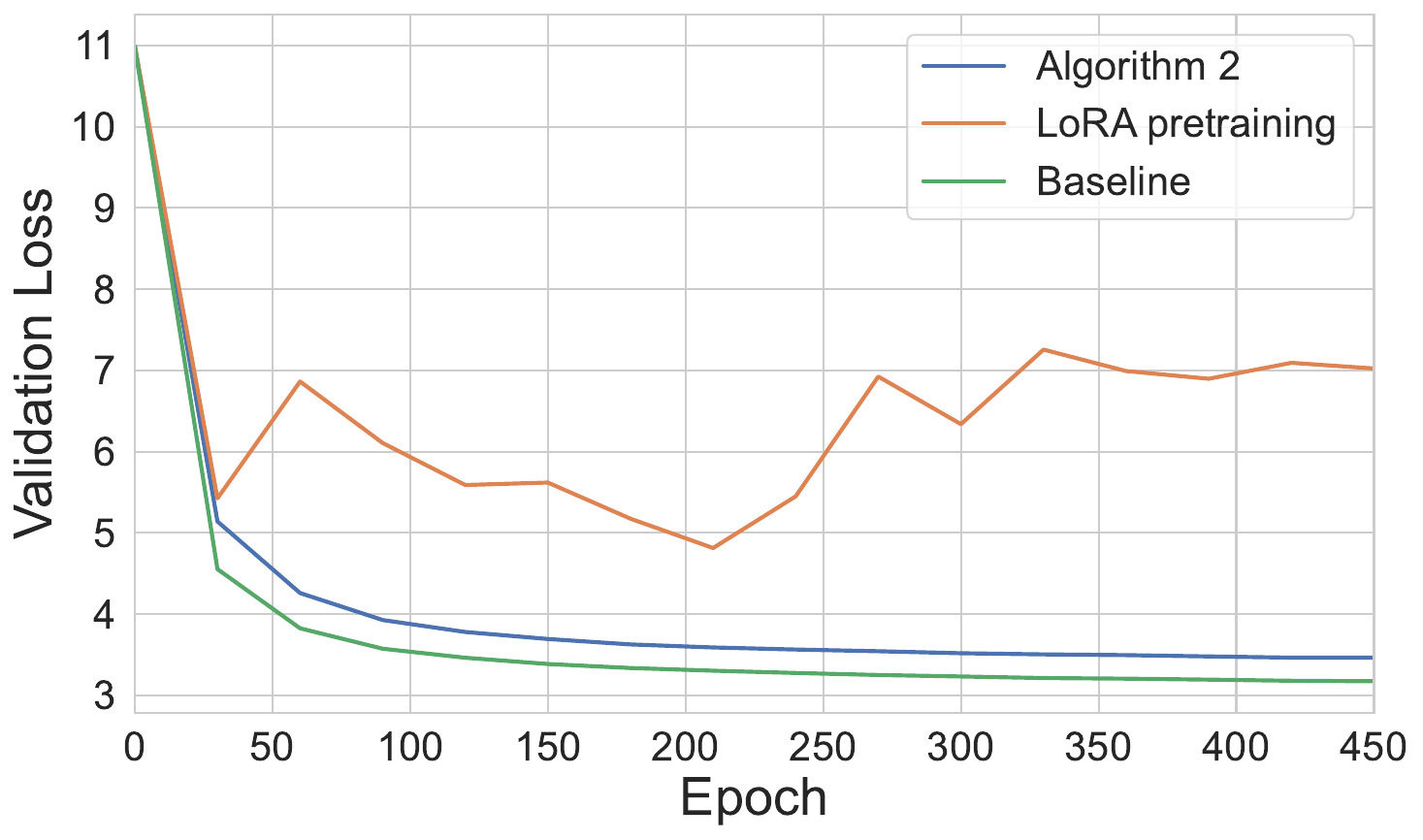}
\end{subfigure}
\caption{GPT2 reproduction on OpenWebText, pretrained from scratch in low-rank, full-rank baseline and \Cref{alg:adam_dlrt} for 15000 iterations.
\Cref{alg:adam_dlrt} method significantly outperforms LoRA pretraining (best validation loss $3.4642$ vs.\ $4.8141$), while incurring only a moderate increase relative to the full-rank baseline ($3.4642$ vs.\ $3.2313$).}\label{fig_GPT2}
\vspace{-0.27cm}
\end{figure}

\subsection{Computational hardware}\label{sec_hardware}
All experiments in this paper are computed using workstation GPUs. Each training run used a single GPU, except for GPT-2 pretraining, which was performed on two NVIDIA H100 GPUs}. Specifically, we have used 5 NVIDIA RTX A6000, 3 NVIDIA RTX 4090, {and 2 NVIDIA H100.}

\newpage

\section*{NeurIPS Paper Checklist}

The checklist is designed to encourage best practices for responsible machine learning research, addressing issues of reproducibility, transparency, research ethics, and societal impact. Do not remove the checklist: {\bf The papers not including the checklist will be desk rejected.} The checklist should follow the references and follow the (optional) supplemental material.  The checklist does NOT count towards the page
limit. 

Please read the checklist guidelines carefully for information on how to answer these questions. For each question in the checklist:
\begin{itemize}
    \item You should answer \answerYes{}, \answerNo{}, or \answerNA{}.
    \item \answerNA{} means either that the question is Not Applicable for that particular paper or the relevant information is Not Available.
    \item Please provide a short (1–2 sentence) justification right after your answer (even for NA). 
\end{itemize}

{\bf The checklist answers are an integral part of your paper submission.} They are visible to the reviewers, area chairs, senior area chairs, and ethics reviewers. You will be asked to also include it (after eventual revisions) with the final version of your paper, and its final version will be published with the paper.

The reviewers of your paper will be asked to use the checklist as one of the factors in their evaluation. While "\answerYes{}" is generally preferable to "\answerNo{}", it is perfectly acceptable to answer "\answerNo{}" provided a proper justification is given (e.g., "error bars are not reported because it would be too computationally expensive" or "we were unable to find the license for the dataset we used"). In general, answering "\answerNo{}" or "\answerNA{}" is not grounds for rejection. While the questions are phrased in a binary way, we acknowledge that the true answer is often more nuanced, so please just use your best judgment and write a justification to elaborate. All supporting evidence can appear either in the main paper or the supplemental material, provided in appendix. If you answer \answerYes{} to a question, in the justification please point to the section(s) where related material for the question can be found.

IMPORTANT, please:
\begin{itemize}
    \item {\bf Delete this instruction block, but keep the section heading ``NeurIPS paper checklist"},
    \item  {\bf Keep the checklist subsection headings, questions/answers and guidelines below.}
    \item {\bf Do not modify the questions and only use the provided macros for your answers}.
\end{itemize}


\begin{enumerate}

\item {\bf Claims}
    \item[] Question: Do the main claims made in the abstract and introduction accurately reflect the paper's contributions and scope?
    \item[] Answer: \answerYes{} 
    \item[] Justification: We describe our contribution in the introduction section.
    \item[] Guidelines:
    \begin{itemize}
        \item The answer NA means that the abstract and introduction do not include the claims made in the paper.
        \item The abstract and/or introduction should clearly state the claims made, including the contributions made in the paper and important assumptions and limitations. A No or NA answer to this question will not be perceived well by the reviewers. 
        \item The claims made should match theoretical and experimental results, and reflect how much the results can be expected to generalize to other settings. 
        \item It is fine to include aspirational goals as motivation as long as it is clear that these goals are not attained by the paper. 
    \end{itemize}

\item {\bf Limitations}
    \item[] Question: Does the paper discuss the limitations of the work performed by the authors?
    \item[] Answer: \answerYes{} 
    \item[] Justification: We discuss the underlying assumptions of the method in the analysis of the corresponding theorems.
    \item[] Guidelines:
    \begin{itemize}
        \item The answer NA means that the paper has no limitation while the answer No means that the paper has limitations, but those are not discussed in the paper. 
        \item The authors are encouraged to create a separate "Limitations" section in their paper.
        \item The paper should point out any strong assumptions and how robust the results are to violations of these assumptions (e.g., independence assumptions, noiseless settings, model well-specification, asymptotic approximations only holding locally). The authors should reflect on how these assumptions might be violated in practice and what the implications would be.
        \item The authors should reflect on the scope of the claims made, e.g., if the approach was only tested on a few datasets or with a few runs. In general, empirical results often depend on implicit assumptions, which should be articulated.
        \item The authors should reflect on the factors that influence the performance of the approach. For example, a facial recognition algorithm may perform poorly when image resolution is low or images are taken in low lighting. Or a speech-to-text system might not be used reliably to provide closed captions for online lectures because it fails to handle technical jargon.
        \item The authors should discuss the computational efficiency of the proposed algorithms and how they scale with dataset size.
        \item If applicable, the authors should discuss possible limitations of their approach to address problems of privacy and fairness.
        \item While the authors might fear that complete honesty about limitations might be used by reviewers as grounds for rejection, a worse outcome might be that reviewers discover limitations that aren't acknowledged in the paper. The authors should use their best judgment and recognize that individual actions in favor of transparency play an important role in developing norms that preserve the integrity of the community. Reviewers will be specifically instructed to not penalize honesty concerning limitations.
    \end{itemize}

\item {\bf Theory Assumptions and Proofs}
    \item[] Question: For each theoretical result, does the paper provide the full set of assumptions and a complete (and correct) proof?
    \item[] Answer: \answerYes{} 
    \item[] Justification:  We discuss the underlying assumptions of the method in the analysis of the corresponding theorems.
    \item[] Guidelines:
    \begin{itemize}
        \item The answer NA means that the paper does not include theoretical results. 
        \item All the theorems, formulas, and proofs in the paper should be numbered and cross-referenced.
        \item All assumptions should be clearly stated or referenced in the statement of any theorems.
        \item The proofs can either appear in the main paper or the supplemental material, but if they appear in the supplemental material, the authors are encouraged to provide a short proof sketch to provide intuition. 
        \item Inversely, any informal proof provided in the core of the paper should be complemented by formal proofs provided in appendix or supplemental material.
        \item Theorems and Lemmas that the proof relies upon should be properly referenced. 
    \end{itemize}

    \item {\bf Experimental Result Reproducibility}
    \item[] Question: Does the paper fully disclose all the information needed to reproduce the main experimental results of the paper to the extent that it affects the main claims and/or conclusions of the paper (regardless of whether the code and data are provided or not)?
    \item[] Answer: \answerYes{} 
    \item[] Justification: All experimental details are provided in the appendix.
    \item[] Guidelines:
    \begin{itemize}
        \item The answer NA means that the paper does not include experiments.
        \item If the paper includes experiments, a No answer to this question will not be perceived well by the reviewers: Making the paper reproducible is important, regardless of whether the code and data are provided or not.
        \item If the contribution is a dataset and/or model, the authors should describe the steps taken to make their results reproducible or verifiable. 
        \item Depending on the contribution, reproducibility can be accomplished in various ways. For example, if the contribution is a novel architecture, describing the architecture fully might suffice, or if the contribution is a specific model and empirical evaluation, it may be necessary to either make it possible for others to replicate the model with the same dataset, or provide access to the model. In general. releasing code and data is often one good way to accomplish this, but reproducibility can also be provided via detailed instructions for how to replicate the results, access to a hosted model (e.g., in the case of a large language model), releasing of a model checkpoint, or other means that are appropriate to the research performed.
        \item While NeurIPS does not require releasing code, the conference does require all submissions to provide some reasonable avenue for reproducibility, which may depend on the nature of the contribution. For example
        \begin{enumerate}
            \item If the contribution is primarily a new algorithm, the paper should make it clear how to reproduce that algorithm.
            \item If the contribution is primarily a new model architecture, the paper should describe the architecture clearly and fully.
            \item If the contribution is a new model (e.g., a large language model), then there should either be a way to access this model for reproducing the results or a way to reproduce the model (e.g., with an open-source dataset or instructions for how to construct the dataset).
            \item We recognize that reproducibility may be tricky in some cases, in which case authors are welcome to describe the particular way they provide for reproducibility. In the case of closed-source models, it may be that access to the model is limited in some way (e.g., to registered users), but it should be possible for other researchers to have some path to reproducing or verifying the results.
        \end{enumerate}
    \end{itemize}

\item {\bf Open access to data and code}
    \item[] Question: Does the paper provide open access to the data and code, with sufficient instructions to faithfully reproduce the main experimental results, as described in supplemental material?
    \item[] Answer: \answerNo{} 
    \item[] Justification:  We provide the open source code upon paper acceptance
    \item[] Guidelines:
    \begin{itemize}
        \item The answer NA means that paper does not include experiments requiring code.
        \item Please see the NeurIPS code and data submission guidelines (\url{https://nips.cc/public/guides/CodeSubmissionPolicy}) for more details.
        \item While we encourage the release of code and data, we understand that this might not be possible, so “No” is an acceptable answer. Papers cannot be rejected simply for not including code, unless this is central to the contribution (e.g., for a new open-source benchmark).
        \item The instructions should contain the exact command and environment needed to run to reproduce the results. See the NeurIPS code and data submission guidelines (\url{https://nips.cc/public/guides/CodeSubmissionPolicy}) for more details.
        \item The authors should provide instructions on data access and preparation, including how to access the raw data, preprocessed data, intermediate data, and generated data, etc.
        \item The authors should provide scripts to reproduce all experimental results for the new proposed method and baselines. If only a subset of experiments are reproducible, they should state which ones are omitted from the script and why.
        \item At submission time, to preserve anonymity, the authors should release anonymized versions (if applicable).
        \item Providing as much information as possible in supplemental material (appended to the paper) is recommended, but including URLs to data and code is permitted.
    \end{itemize}

\item {\bf Experimental Setting/Details}
    \item[] Question: Does the paper specify all the training and test details (e.g., data splits, hyperparameters, how they were chosen, type of optimizer, etc.) necessary to understand the results?
    \item[] Answer: \answerYes{} 
    \item[] Justification: We provide the training details and hyperparameters in the appendix.
    \item[] Guidelines:
    \begin{itemize}
        \item The answer NA means that the paper does not include experiments.
        \item The experimental setting should be presented in the core of the paper to a level of detail that is necessary to appreciate the results and make sense of them.
        \item The full details can be provided either with the code, in appendix, or as supplemental material.
    \end{itemize}

\item {\bf Experiment Statistical Significance}
    \item[] Question: Does the paper report error bars suitably and correctly defined or other appropriate information about the statistical significance of the experiments?
    \item[] Answer: \answerYes{} 
    \item[] Justification: We provide error bars and report the mean and median over different initializations.
    \item[] Guidelines:
    \begin{itemize}
        \item The answer NA means that the paper does not include experiments.
        \item The authors should answer "Yes" if the results are accompanied by error bars, confidence intervals, or statistical significance tests, at least for the experiments that support the main claims of the paper.
        \item The factors of variability that the error bars are capturing should be clearly stated (for example, train/test split, initialization, random drawing of some parameter, or overall run with given experimental conditions).
        \item The method for calculating the error bars should be explained (closed form formula, call to a library function, bootstrap, etc.)
        \item The assumptions made should be given (e.g., Normally distributed errors).
        \item It should be clear whether the error bar is the standard deviation or the standard error of the mean.
        \item It is OK to report 1-sigma error bars, but one should state it. The authors should preferably report a 2-sigma error bar than state that they have a 96\% CI, if the hypothesis of Normality of errors is not verified.
        \item For asymmetric distributions, the authors should be careful not to show in tables or figures symmetric error bars that would yield results that are out of range (e.g. negative error rates).
        \item If error bars are reported in tables or plots, The authors should explain in the text how they were calculated and reference the corresponding figures or tables in the text.
    \end{itemize}

\item {\bf Experiments Compute Resources}
    \item[] Question: For each experiment, does the paper provide sufficient information on the computer resources (type of compute workers, memory, time of execution) needed to reproduce the experiments?
    \item[] Answer: \answerYes{} 
    \item[] Justification: The used compute resources are reported in the appendix
    \item[] Guidelines:
    \begin{itemize}
        \item The answer NA means that the paper does not include experiments.
        \item The paper should indicate the type of compute workers CPU or GPU, internal cluster, or cloud provider, including relevant memory and storage.
        \item The paper should provide the amount of compute required for each of the individual experimental runs as well as estimate the total compute. 
        \item The paper should disclose whether the full research project required more compute than the experiments reported in the paper (e.g., preliminary or failed experiments that didn't make it into the paper). 
    \end{itemize}
    
\item {\bf Code Of Ethics}
    \item[] Question: Does the research conducted in the paper conform, in every respect, with the NeurIPS Code of Ethics \url{https://neurips.cc/public/EthicsGuidelines}?
    \item[] Answer: \answerYes{} 
    \item[] Justification: The paper conforms, in every respect, with the NeurIPS Code of Ethics
    \item[] Guidelines:
    \begin{itemize}
        \item The answer NA means that the authors have not reviewed the NeurIPS Code of Ethics.
        \item If the authors answer No, they should explain the special circumstances that require a deviation from the Code of Ethics.
        \item The authors should make sure to preserve anonymity (e.g., if there is a special consideration due to laws or regulations in their jurisdiction).
    \end{itemize}

\item {\bf Broader Impacts}
    \item[] Question: Does the paper discuss both potential positive societal impacts and negative societal impacts of the work performed?
    \item[] Answer: \answerYes{} 
    \item[] Justification: The impacts are discussed in the conclusion
    \item[] Guidelines:
    \begin{itemize}
        \item The answer NA means that there is no societal impact of the work performed.
        \item If the authors answer NA or No, they should explain why their work has no societal impact or why the paper does not address societal impact.
        \item Examples of negative societal impacts include potential malicious or unintended uses (e.g., disinformation, generating fake profiles, surveillance), fairness considerations (e.g., deployment of technologies that could make decisions that unfairly impact specific groups), privacy considerations, and security considerations.
        \item The conference expects that many papers will be foundational research and not tied to particular applications, let alone deployments. However, if there is a direct path to any negative applications, the authors should point it out. For example, it is legitimate to point out that an improvement in the quality of generative models could be used to generate deepfakes for disinformation. On the other hand, it is not needed to point out that a generic algorithm for optimizing neural networks could enable people to train models that generate Deepfakes faster.
        \item The authors should consider possible harms that could arise when the technology is being used as intended and functioning correctly, harms that could arise when the technology is being used as intended but gives incorrect results, and harms following from (intentional or unintentional) misuse of the technology.
        \item If there are negative societal impacts, the authors could also discuss possible mitigation strategies (e.g., gated release of models, providing defenses in addition to attacks, mechanisms for monitoring misuse, mechanisms to monitor how a system learns from feedback over time, improving the efficiency and accessibility of ML).
    \end{itemize}
    
\item {\bf Safeguards}
    \item[] Question: Does the paper describe safeguards that have been put in place for responsible release of data or models that have a high risk for misuse (e.g., pretrained language models, image generators, or scraped datasets)?
    \item[] Answer: \answerNA{} 
    \item[] Justification: This work is algorithmic and does not release special data or models.
    \item[] Guidelines:
    \begin{itemize}
        \item The answer NA means that the paper poses no such risks.
        \item Released models that have a high risk for misuse or dual-use should be released with necessary safeguards to allow for controlled use of the model, for example by requiring that users adhere to usage guidelines or restrictions to access the model or implementing safety filters. 
        \item Datasets that have been scraped from the Internet could pose safety risks. The authors should describe how they avoided releasing unsafe images.
        \item We recognize that providing effective safeguards is challenging, and many papers do not require this, but we encourage authors to take this into account and make a best faith effort.
    \end{itemize}

\item {\bf Licenses for existing assets}
    \item[] Question: Are the creators or original owners of assets (e.g., code, data, models), used in the paper, properly credited and are the license and terms of use explicitly mentioned and properly respected?
    \item[] Answer: \answerNA{} 
    \item[] Justification: Not needed for this work
    \item[] Guidelines:
    \begin{itemize}
        \item The answer NA means that the paper does not use existing assets.
        \item The authors should cite the original paper that produced the code package or dataset.
        \item The authors should state which version of the asset is used and, if possible, include a URL.
        \item The name of the license (e.g., CC-BY 4.0) should be included for each asset.
        \item For scraped data from a particular source (e.g., website), the copyright and terms of service of that source should be provided.
        \item If assets are released, the license, copyright information, and terms of use in the package should be provided. For popular datasets, \url{paperswithcode.com/datasets} has curated licenses for some datasets. Their licensing guide can help determine the license of a dataset.
        \item For existing datasets that are re-packaged, both the original license and the license of the derived asset (if it has changed) should be provided.
        \item If this information is not available online, the authors are encouraged to reach out to the asset's creators.
    \end{itemize}

\item {\bf New Assets}
    \item[] Question: Are new assets introduced in the paper well documented and is the documentation provided alongside the assets?
    \item[] Answer: \answerNA{} 
    \item[] Justification: Not needed for this work
    \item[] Guidelines:
    \begin{itemize}
        \item The answer NA means that the paper does not release new assets.
        \item Researchers should communicate the details of the dataset/code/model as part of their submissions via structured templates. This includes details about training, license, limitations, etc. 
        \item The paper should discuss whether and how consent was obtained from people whose asset is used.
        \item At submission time, remember to anonymize your assets (if applicable). You can either create an anonymized URL or include an anonymized zip file.
    \end{itemize}

\item {\bf Crowdsourcing and Research with Human Subjects}
    \item[] Question: For crowdsourcing experiments and research with human subjects, does the paper include the full text of instructions given to participants and screenshots, if applicable, as well as details about compensation (if any)? 
    \item[] Answer: \answerNA{} 
    \item[] Justification:  Not needed for this work
    \item[] Guidelines:
    \begin{itemize}
        \item The answer NA means that the paper does not involve crowdsourcing nor research with human subjects.
        \item Including this information in the supplemental material is fine, but if the main contribution of the paper involves human subjects, then as much detail as possible should be included in the main paper. 
        \item According to the NeurIPS Code of Ethics, workers involved in data collection, curation, or other labor should be paid at least the minimum wage in the country of the data collector. 
    \end{itemize}

\item {\bf Institutional Review Board (IRB) Approvals or Equivalent for Research with Human Subjects}
    \item[] Question: Does the paper describe potential risks incurred by study participants, whether such risks were disclosed to the subjects, and whether Institutional Review Board (IRB) approvals (or an equivalent approval/review based on the requirements of your country or institution) were obtained?
    \item[] Answer: \answerNA{} 
    \item[] Justification:  Not needed for this work
    \item[] Guidelines:
    \begin{itemize}
        \item The answer NA means that the paper does not involve crowdsourcing nor research with human subjects.
        \item Depending on the country in which research is conducted, IRB approval (or equivalent) may be required for any human subjects research. If you obtained IRB approval, you should clearly state this in the paper. 
        \item We recognize that the procedures for this may vary significantly between institutions and locations, and we expect authors to adhere to the NeurIPS Code of Ethics and the guidelines for their institution. 
        \item For initial submissions, do not include any information that would break anonymity (if applicable), such as the institution conducting the review.
    \end{itemize}

\end{enumerate}

\end{document}